%% file: paper.tex
\documentclass{article}
\usepackage[accepted]{icml2020}
\usepackage{tikz}
\usetikzlibrary{matrix}
\usepackage{url}            
\usepackage{booktabs}       
\usepackage{amsfonts}       
\usepackage{float}
\usepackage{natbib} 
\setcitestyle{round} 

\input{./Definitions}
\usepackage{mdwlist}
\usepackage{tabularx}
\usepackage{multirow}
\usepackage{epsf}
\usepackage{amsmath,amssymb}
\usepackage{thmtools}
\usepackage{thm-restate}
\usepackage{mathrsfs}
\usepackage{xspace}
\usepackage{graphicx}
\usepackage{subfig}
\usepackage{wrapfig}
\usepackage[novbox]{pdfsync}
\usepackage[colorlinks,linkcolor=red,citecolor=blue,urlcolor=black]{hyperref}

\usepackage{cleveref}
\usepackage{enumitem}
\usepackage{tikz-cd}

\usetikzlibrary{babel}
\tikzcdset{every label/.append style = {font = \normalsize}}

\newcommand{\dev}[1]{\scriptstyle \pm #1}

\newcommand{\dual}[2]{\left\langle #1; #2 \right\rangle}
\newcommand{\Bsf}{\mathsf{B}}
\newcommand{\Bsftil}{\tilde{\mathsf{B}}}
\newcommand{\Fsf}{\mathsf{F}}
\newcommand{\ev}{\mathsf{ev}}
\newcounter{comment}

\newif\ifproof
\prooffalse

\ifproof
\usepackage[textsize=tiny]{todonotes}
\else
\usepackage[disable]{todonotes}

\fi

\graphicspath{ {images/} }
\icmltitlerunning{Convex Representation Learning for Generalized Invariance in  Semi-Inner-Product Space}

\begin{document}
\twocolumn[
\icmltitle{Convex Representation Learning for Generalized Invariance in \\ Semi-Inner-Product Space}

\icmlsetsymbol{equal}{*}
\begin{icmlauthorlist}
  \icmlauthor{Yingyi Ma }{a}
  \icmlauthor{Vignesh Ganapathiraman}{a}
  \icmlauthor{Yaoliang Yu }{b}
  \icmlauthor{Xinhua Zhang }{a}
\end{icmlauthorlist}
\icmlaffiliation{a}{University of Illinois at Chicago}
\icmlaffiliation{b}{University of Waterloo and Vector Institute}

\icmlcorrespondingauthor{Xinhua Zhang}{zhangx@uic.edu}

\vspace{3em}
]
\printAffiliationsAndNotice{}
\begin{abstract}
   Invariance (defined in a general sense) has been one of the most effective priors for representation learning.
  Direct factorization of parametric models is feasible only for a small range of invariances,
  while regularization approaches, despite improved generality,
  lead to nonconvex optimization.
  In this work, we develop a \emph{convex} representation learning algorithm  for a variety of generalized invariances that can be modeled as semi-norms.  
  Novel Euclidean embeddings are introduced for kernel representers in a semi-inner-product space,
  and approximation bounds are established.
  This allows invariant representations to be learned efficiently and effectively as confirmed in our experiments, along with accurate predictions.
\end{abstract}

\section{Introduction}

Effective modeling of structural priors has been the workhorse of a variety of machine learning algorithms.
Such priors are available in a rich supply, 
including invariance \citep{SimardLDV96,FerCae94}, 
equivariance \citep{CohWel16,GraRav19}, 
disentanglement \citep{BenCouVin13,Higginsetal17},
homophily/heterophily \citep{EliFal12}, 
fairness \citep{Creageretal19}, 
correlations in multiple views and modalities \citep{WanAroLivBil15,Kumaretal18}, etc.

In this paper we focus on ``generalized invariance'', 
where certain relationship holds irrespective of certain changes in data.
This extends traditional settings that are limited to, \eg, transformation and permutation.
For instance, in multilabel classification there are semantic or logical relationships between classes which hold for any input.
Common examples include mutual exclusion and implication \citep{MirRavDinSch15,Dengetal12}.
In mixup \citep{ZhaCisDauLop18}, 
a convex interpolation of a pair of examples is postulated to yield the same interpolation of output labels.

While conventional wisdom learns models whose prediction accords with these  structures,
recent developments show that it can be more effective to learn structure-encoding representations.
Towards this goal, 
the most straightforward approach is to directly parameterize the model.
For example, deep sets model permutation invariance via an additive decomposition \citep{Zaheeretal17},
convolutional networks use sparse connection and parameter sharing to model translational invariance,
and a similar approach has been developed for equivariance \citep{RavSchPoc17}.
Although they simplify the model and can enforce invariance over the \emph{entire} space,
their applicability is very restricted, 
because most useful structures do not admit a readily decomposable parameterization.
As a result, most invariance/equivariance models are restricted to permutations and group based diffeomorphism.

In order to achieve significantly improved generality and flexibility,
the regularization approach can be leveraged,
which penalizes the violation of pre-specified structures.
For example, \citet{Rifaietal11} penalizes the norm of the Jacobian matrix to enforce contractivity, conceivably a generalized type of invariance.
\citet{Smola19} proposed using a max-margin loss over all transformations \citep{TeoGloRowSmo07}.
However, for most structures, regularization leads to a nonconvex problem.
Despite the recent progress in optimization for deep learning,
the process still requires a lot of trial and error.
Therefore a convex learning procedure will be desirable, 
because besides the convenience in optimization, 
it also offers the profound advantage of decoupling parameter optimization from problem specification: 
poor learning performance can only be ascribed to a poor model architecture, 
not to poor local minima.

Indeed convex invariant representation learning has been studied, 
but in limited settings.
Tangent distance kernels \citep{HaaKey02} and
Haar integration kernels are engineered to be invariant to a group of transformations \cite{Rajetal17,MroVoiPog15,HaaBur07},
but it relies on sampling for tractable computation and the sample complexity is $O(d/\epsilon^2)$ where $d$ is the dimension of the underlying space.
\citet{BhaPanSmo05} treated all perturbations within an ellipsoid neighborhood as invariances,
and it led to an expensive second order cone program (SOCP).  
Other distributionally robust formulations also lead to SOCP/SDPs \citep{RahMeh19}.
The most related work is \citet{MaGanZha19},
which warped a reproducing kernel Hilbert space (RKHS) by linear functionals that encode the invariances.
However, in order to keep the warped space an RKHS,
their applicability is restricted to \emph{quadratic} losses on linear functionals.

In practice, however, there are many invariances that cannot be modeled by quadratic penalties.
For example, the logical relationships between classes impose an ordering in the discriminative output \citep{MirRavDinSch15},
and this can hardly be captured by quadratic forms.
Similarly, when a large or infinite number of invariances are available, 
measuring the maximum violation makes more sense than their sum,
and it is indeed the principle underlying adversarial learning \cite{MadMakSchetal18}.
Again this is not amenable to quadratic forms.

Our goal, therefore, is to develop a \emph{convex} representation learning approach that efficiently incorporates generalized invariances as semi-norm functionals.
Our first contribution is to show that compared with linear functionals,
semi-norm functionals encompass a much broader range of invariance (Sections \ref{sec:mixup} and \ref{sec:multilabel}).

Our key tool is the semi-inner-product space \citep[\sip,][]{Lumer61},
into which an RKHS can be warped by augmenting the RKHS norm with semi-norm functionals.
A specific example of \sip\ space is 
the reproducing kernel Banach space \citep{ZhaXuZha09},
which has been used for $\ell_p$ regularization in, 
\eg, kernel SVMs,
and suffers from high computational cost
\citep{SalSuyRos18,DerLee07,BenBre00,HeiBouSch05,LuxBou04,ZhouXiaoZhoDai02}.
A \sip\ space extends RKHS by relaxing the underlying inner product into a semi-inner-product,
while retaining the important construct: \emph{kernel function}.
To our best knowledge, 
\sip\ space has yet been applied to representation learning.

Secondly, we developed efficient computation algorithms for solving the regularized risk minimization (RRM) with the new \sip\ norm (Section \ref{sec:RRM}).
Although \citet{ZhaXuZha09} established the representer theorem from a pure mathematical perspective,
no practical algorithm was provided and ours is the first to fill this gap.

However, even with this progress,
RRMs still do not provide invariant representations of data instances;
it simply learns a discriminant function by leveraging the representer theorem 
(which does hold in the applications we consider).
So our third contribution, as presented in Section \ref{sec:euc_embed},
is to learn and extract representations by embedding s.i.p. kernel representers in Euclidean spaces.
This is accomplished in a \emph{convex} and efficient fashion,
constituting a secondary advantage over RRMs which is not convex in the dual coefficients.
Different from Nystr\"om or Fourier linearization of kernels in RKHS,
the kernel representers in a \sip\ space carry interestingly different meanings and expressions in primal and dual spaces.
Finally, our experiments demonstrate that the new \sip-based algorithm learns more predictive representations than strong baselines.

\section{Preliminaries}

Suppose we have an RKHS $\Hcal = (\Fcal, \inner{\cdot}{\cdot}_{\Hcal}, k)$ with $\Fcal \subseteq \RR^{\Xcal}$, inner product $\inner{\cdot}{\cdot}_{\Hcal}$ and  kernel $k: \Xcal \times \Xcal \to \RR$.
Our goal is to renorm $\Hcal$ hence \emph{warp the distance metric} by adding a functional $R$ that induces desired structures.

\subsection{Existing works on invariance modeling by RKHS}

\citet{SmoSch98} and \citet{ZhaLeeTeh2013} proposed modeling invariances by bounded linear functionals in RKHS.
Given a function $f$, the graph Laplacian is
$\sum_{ij} w_{ij} (f(x_i)  -  f(x_j))^2$,
and obviously $f(x_i)  -  f(x_j)$ is bounded and linear.
Transformation invariance can be characterized by 
$\frac{\partial}{\partial \alpha} |_{\alpha=0} f(I(\alpha))$,
where $I(\alpha)$ stands for the image after applying an $\alpha$ amount of rotation, translation, etc.
It is again bounded and linear.
By Riesz representation theorem,
a bounded linear functional can be written as $\inner{z_i}{f}_\Hcal$ for some $z_i \in  \Hcal$.

Based on this view,
\citet{MaGanZha19} took a step towards representation learning.
By adding $R(f)^2 := \sum_{i} \inner{z_i}{f}^2_\Hcal$ to the RKHS norm square,
the space is warped to favor $f$ that respects invariance,
\ie, small magnitude of $\inner{z_i}{f}$.
They showed that it leads to a new RKHS with a kernel 
%
\begin{align}
\label{eq:rkhs_warp_kernel}
	k^\circ(x_1, x_2) = k(x_1, x_2) - z(x_1)^\top (I+K_z)^{-1} z(x_2),
\end{align}
%
where $z(x) = \! (z_1(x), \ldots, z_m(x))^\top$ and $K_z \! = (\inner{z_i}{z_j})_{i,j}$.

Although the kernel representer of $k^\circ$ offers a new invariance aware representation,
the requirement that the resulting space remains an RKHS forces the penalties in $R$ to be quadratic on $\inner{z_i}{f}$,
significantly limiting its applicability to a broader range of invariances such as total variation $\int_x |f'(x)| \rmd x$.
Our goal is to relax this restriction by enabling \emph{semi-norm} regularizers with new tools in functional analysis,
and illustrate its applications in Sections \ref{sec:mixup} and \ref{sec:multilabel}.

\subsection{Semi-inner-product spaces}

We first specify the range of regularizer $R$ considered here.

\begin{assumption}
\label{assump:Reg}
We assume that $R: \Fcal \to \RR$ is a  semi-norm.
Equivalently, $R: \Fcal \to \RR$ is convex and $R(\alpha f) = \abr{\alpha} R(f)$ for all $f \in \Fcal$ and $\alpha \in \RR$ (absolute homogeneity).
Furthermore, we assume $R$ is closed (i.e., lower semicontinuous) \wrt the topology in $\Hcal$.
\end{assumption}
Since $R$ is closed convex and its domain is the entire Hilbert space $\Hcal$,
$R$ must be continuous.
By exempting $R$ from being induced by an inner product,
we enjoy substantially improved flexibility in modeling various  regularities.

For most learning tasks addressed below,
it will be convenient to directly construct $R$ from the specific regularity.
However, in some context it will also be convenient to constructively explicate $R$ in terms of support functions.

\begin{proposition}
	\label{prop:reg_R_supfun}
	$R(f)$ satisfies Assumption \ref{assump:Reg} if, and only if,
  $R(f) = \sup_{g \in S} \inner{f}{g}_\Hcal$,
    where $S \subseteq \Hcal$ is bounded in the RKHS norm
    and is symmetric ($g \in S \Leftrightarrow -g \in S$).    
\end{proposition}


%
The proof is in Appendix \ref{app:proof}.
Using $R$, we arrive at a new norm defined by
%
\begin{align}
\label{eq:norm_sublinear}
     \nbr{f}_\Bcal := \sqrt{\nbr{f}_\Hcal^2 + R(f)^2},
\end{align}
thanks to Assumption \ref{assump:Reg}. 
It is immediately clear from \Cref{prop:reg_R_supfun} that 
$   \nbr{f}_\Hcal \leq \nbr{f}_\Bcal \leq C \nbr{f}_\Hcal$, 
for some constant $C > 0$ that bounds the norm of $S$. In other words, the two norms $\nbr{\cdot}_\Hcal$ and $\nbr{\cdot}_\Bcal$ are equivalent, hence in particular the norm $\nbr{\cdot}_\Bcal$ is complete. We thus arrive at a Banach space $\Bcal = (\Fcal, \|\cdot\|_{\Bcal})$. Note that both $\Hcal$ and $\Bcal$ have the same underlying vector space $\Fcal$---the difference is in the norm or distance metric.
To proceed, we need to endow more structures on $\Bcal$.



\begin{definition}[Strict convexity]
	A normed vector space $(\Fcal, \|\cdot\|)$ is strictly convex if for all $\zero\ne f, g \in \Fcal$, 
%
	\begin{align}
	\nbr{f + g} = \nbr{f} + \nbr{g}
	\end{align}
%
	implies $g = \alpha f$ for some $\alpha \geq0$. 
	Equivalently, if the unit ball $\Bsf := \{f \in \Fcal: \|f\| \leq 1\}$ is strictly convex.
\end{definition}


Using the parallelogram law it is clear that the Hilbert norm $\|\cdot\|_{\Hcal}$ is strictly convex. Moreover, since summation preserves strict convexity, it follows that the new norm $\|\cdot\|_{\Bcal}$ is strictly convex as well.




\begin{definition}[G\^ateaux differentiability] A normed vector space $(\Fcal, \|\cdot\|)$ is G\^ateaux differentiable if for all $\zero \ne f, g \in \Fcal$, there exists the directional derivative   
%
\begin{align}
  \lim\nolimits_{t \in \RR, t \to 0} \smallfrac{1}{t} (\nbr{f+tg} - \nbr{f}).
\end{align}
%

\end{definition}

We remark that both strict convexity and G\^ateaux differentiability are algebraic but not topological properties of the norm. In other words, two equivalent (in terms of topology) norms may not be strictly convex or G\^ateaux differentiable at the same time. For instance, the $\ell_2$-norm on $\RR^d$ is both strictly convex and G\^ateaux differentiable, while the equivalent $\ell_1$-norm is not.

Recall that $\Bcal^*$ is the dual space of $\Bcal$, consisting of all continuous linear functionals on $\Bcal$ and equipped with the dual norm $\|F\|_{\Bcal^*} = \sup_{\nbr{f}_\Bcal \le 1 } F(f)$. 
The dual space of a normed (reflexive) space is Banach (reflexive).

\begin{definition}
    A Banach space $\Bcal$ is reflexive if the canonical map $\jmath: \Bcal \to \Bcal^{**}$, $f \mapsto \dual{\cdot}{\jmath f} := \dual{f}{\cdot}$ is onto, where $\dual{f}{F}$ is the (bilinear) duality pairing between dual spaces. Here $\cdot$ is any element in $\Bcal^*$.
\end{definition}

Note that reflexivity is a topological property. 
In particular, equivalent norms are all reflexive if any one of them is. 
As any Hilbert space $\Hcal$ is reflexive, so is the equivalent norm $\|\cdot\|_{\Bcal}$ in \eqref{eq:norm_sublinear}.

\begin{theorem}[{\citealt[p. 212-213]{BorwenVanderwerff10}}]
    A Banach space $\Bcal$ is strictly convex (G\^ateaux differentiable) if its dual space $\Bcal^*$ is G\^ateaux differentiable (strictly convex). 
    The converse is true too if $\Bcal$ is reflexive.
    \label{thm:ref}
\end{theorem}
Combining \Cref{prop:reg_R_supfun} and \Cref{thm:ref}, we see that $R$, hence $\|\cdot\|_{\Bcal}$, is G\^ateaux differentiable if (the closed convex hull of) the set $S$ in \Cref{prop:reg_R_supfun} is strictly convex.

We are now ready to define 
a semi-inner-product (s.i.p.) on a normed space $(\Fcal, \|\cdot\|)$. We call a bivariate mapping   
$[\cdot, \cdot] : \Fcal \times \Fcal \to \RR$ a s.i.p. if for all $f, g, h \in \Fcal$ and $\lambda \in\RR$,
\begin{itemize}[topsep=0pt,itemsep=0pt]
    \item additivity: $[f+g, h] = [f, h] + [g,h]$
    \item homogeneity: $[\lambda f, g] = [f, \lambda g] = \lambda [f, g]$,
    \item norm-inducing: $[f, f] = \|f\|^2$,
    \item Cauchy-Schwarz: $[f,g] \le \|f\| \cdot \|g\|$.
\end{itemize}
We note that an s.i.p. is additive in its second argument iff it is an inner product (by simply verifying the parellelogram law).
\citet{Lumer61} proved that s.i.p. does exist on every normed space. Indeed, let the subdifferential $J = \partial \tfrac12\|\cdot\|^2_\Bcal : \Bcal \rightrightarrows \Bcal^*$ be the (multi-valued) duality mapping. Then, any selection $j: \Bcal \to \Bcal^*, f\mapsto j(f)\in J(f)$ with the convention that $j(\zero) = \zero$ leads to a s.i.p.: 
%
\begin{align}
\label{eq:sip}
    [f, g] := \dual{f}{j(g)}.
\end{align}
%
Indeed, from definition, for any $f \ne \zero$, $j(f) = \|f\| F$, where $\|F\|^* = 1$ and $\dual{f}{F} = \|f\|$.
A celebrated result due to \citet{Giles67} revealed the uniqueness of s.i.p. if the norm $\|\cdot\|$ is G\^ateaux differentiable, and later \citet{Faulkner77} proved that the (unique) mapping $j$ is onto iff $\Bcal$ is reflexive. Moreover, $j$ is 1-1 if $\Bcal$ is strictly convex (like in \eqref{eq:norm_sublinear}), as was shown originally in \citet{Giles67}. 

Let us summarize the above results. 
\begin{definition}
    A Banach space $\Bcal$ is called a s.i.p. space iff it is reflexive, strictly convex, and G\^ateaux differentiable. Clearly, the dual $\Bcal^*$ of a s.i.p. space is s.i.p. too.
\end{definition}
\begin{theorem}[Riesz representation]
\label{thm:representer_operator}
    Let $\Bcal$ be a s.i.p. space. Then, for any continuous linear functional $f^* \in \Bcal^*$, there exists a unique $f \in \Bcal$ such that 
    %
    \begin{align}
        f^* = [\cdot, f] = j(f), ~~ \mbox{ and } ~~ \|f\| = \|f^*\|_{\Bcal^*}.
    \end{align}
    %
\end{theorem}
From now on, we identify the duality mapping $j$ with the star operator $f^* := j(f)$. 
Thus, we have a unique way to represent all continuous functionals on a s.i.p. space.  Conveniently, the unique s.i.p. on the dual space follows from \eqref{eq:sip}: for all $ f^*, g^* \in \Bcal^*$, 
%
\begin{align}
\label{eq:dual_sip_connection}
    [f^*, g^*] := [g, f] = \dual{g}{f^*},
\end{align}
from which one easily verifies all properties of an s.i.p.
Some literature writes $[f^*, g^*]_{\Bcal^*}$, $[g, f]_\Bcal$, $\dual{g}{f^*}_{\Bcal^*}$, and
$\dual{f}{g^*}_\Bcal$ to explicitize where the operations take place.
We simplify these notations by omitting subscripts when the context is clear,
but still write $\nbr{f}_\Bcal$ and $\nbr{f^*}_{\Bcal^*}$.

Finally, fix $x \in \Xcal$ and consider the evaluation (linear) functional $\ev_x: \Bcal \to \RR$, $f \mapsto f(x)$. When $\ev_x$ is continuous (which indeed holds for our norm \eqref{eq:norm_sublinear}), \Cref{thm:representer_operator} implies the existence of a unique $G_x \in \Bcal$ such that 
%
\begin{align}
\label{eq:eval_fx_rkbs}
    f(x) = \ev_x(f) = [f, G_x] = [G_x^*, f^*].
\end{align}
Varying $x \in \Xcal$ we obtain a unique \textit{s.i.p. kernel} $G: \Xcal \times \Xcal \to \RR$ such that 
$G_x := G(\cdot, x) \in \Bcal$. Thus, using s.i.p. we obtain the reproducing property:
%
\begin{align}
    f(x) = [f, G(\cdot, x)], ~ G(x,y) = [G(\cdot, y), G(\cdot, x)].
\end{align}
Different from a reproducing kernel in RKHS, $G$ is not necessarily symmetric or positive semi-definite.



\section{Regularized Risk Minimization}
\label{sec:RRM}

In this section we aim to provide a computational device for the following regularized risk minimization (RRM) problem:
%
\begin{align}
    \min_{f \in \Hcal} \ \ \ell(f) + \nbr{f}_\Hcal^2 + R(f)^2.
\end{align}
%
where $\ell(f)$ is the empirical risk depending on discriminant function values $\{f(x_j)\}_{j=1}^n$ for training examples $\{x_j\}$.
Clearly, this objective is equivalent to
%
\begin{align}
\label{eq:obj_rrm_rkbs}
    \min_{f \in \Bcal} \ \ \ell(f) + \nbr{f}_\Bcal^2.
\end{align}
%

\begin{remark}
    Unlike the usual treatment in reproducing kernel Banach spaces (RKBS) \citep[\eg][]{ZhaXuZha09}, we only require $\Bcal$ to be reflexive, strictly convex and G\^ateaux differentiable, instead of the much more demanding uniform convexity and smoothness. This more general condition not only suffices for our subsequent results but also simplifies the presentation. A similar definition like ours was termed pre-RKBS in \citet{CombettesSV18}.
\end{remark}

\citet[][Theorem 2]{ZhaXuZha09} established the  representer theorem for RKBS:
the optimal $f$ for \eqref{eq:obj_rrm_rkbs} has its dual form
%
\begin{align}
\label{eq:representer_thm}
    f^* 
    = \sum\nolimits_j c_j G_{x_j}^*,
\end{align}
%
where $\{c_j\}$ are real coefficients. 
To optimize $\{c_j\}$, we need to substitute \eqref{eq:representer_thm} into \eqref{eq:obj_rrm_rkbs},
which in turn requires evaluating 
i) $\nbr{f}_\Bcal^2$, which equals $\nbr{f^*}^2_{\Bcal^*}$;
ii) $f(x)$, which, can be computed through inverting the star operator as follows:
%
%
\begin{align*}
\nonumber
    \nbr{f^*}_{\Bcal^*} &= \max_{\|h\|_{\Bcal}\leq 1} \dual{h}{f^*} \\
    &= \max_{\|h\|_{\Bcal} \leq 1} \sum\nolimits_j c_j \dual{h}{ G^*_{x_j}} \\
    &= \max_{h: \nbr{h}_\Hcal^2 + R(h)^2 \le 1} \sum\nolimits_j c_j h(x_j),
\end{align*}
where the last equality is due to \eqref{eq:eval_fx_rkbs} and \eqref{eq:sip}.
The last maximization step operates in the RKHS $\Hcal$, 
and thanks to the strict convexity of $\|\cdot\|_{\Bcal}$,
it admits the unique solution
\begin{align}
    h = f / \|f\|_{\Bcal} = f / \|f^*\|_{\Bcal^*},
\end{align}
because $\dual{f}{f^*} = \|f\|_{\Bcal} \|f^*\|_{\Bcal^*}$, and $\Bcal$ is a s.i.p. space.

%

We summarize this computational inverse below:
\begin{theorem}
	\label{them:fx_unique}
	If $f^* =  \sum_j c_j G_{x_j}^*$,
	then $f = \nbr{f}_\Bcal f^{\circ}$,
	where
	%
	\begin{align}
	\label{eq:def_hopt}
		f^{\circ} &:= \argmax_{h: \nbr{h}_\Hcal^2 + R(h)^2 \le 1} \sum_j c_j h(x_j), \\
		\nbr{f}_\Bcal &= \sum\nolimits_j c_j f^{\circ}(x_j) = \inner{f^\circ}{\sum\nolimits_j c_j k(x_j, \cdot)}_\Hcal.
	\end{align}
	In addition, the argmax in \eqref{eq:def_hopt} is attained uniquely.
\end{theorem}

In practice, we first compute $f^{\circ}$ by solving \eqref{eq:def_hopt},
and then $f$ can be evaluated at different $x$ without redoing any optimization.
As a special case, setting $f^* \! = G_x^*$ allows us to evaluate the kernel 
$    G_x = G^\circ_x(x) G_x^\circ$.

\paragraph{Specialization to RKHS.}
When $R(f)^2 = \sum_{i} \inner{z_i}{f}^2_\Hcal$,
$\nbr{\cdot}_\Bcal$ is induced by an inner product,
making $\Bcal$ an RKHS.
Now we can easily recover \eqref{eq:rkhs_warp_kernel} by applying Theorem \ref{them:fx_unique},
because the optimization in \eqref{eq:def_hopt} with $f^* = G^*_x$ is
%
\begin{align}
	\max_{h \in \Hcal} ~h(x), \quad \mathrm{s.t.} \quad \nbr{h}_\Hcal^2 + \sum\nolimits_{k} \inner{z_k}{h}^2_\Hcal \le 1,
\end{align}
and its unique solution can be easily found in closed form:
%
\begin{align}
	G_x^\circ = \frac{k(\cdot, x)- (z_1, \ldots, z_m) (I + K_z)^{-1} z(x)}{(k(x,x) - z(x)^\top (I+K_z)^{-1} z(x))^{1/2}}.
\end{align}
Plugging into $G_x = G^\circ_x(x) G_x^\circ$,
we recover \eqref{eq:rkhs_warp_kernel}.

Overall, the optimization of \eqref{eq:obj_rrm_rkbs} may no longer be convex in $\{c_j\}$,
because $f(x)$ is generally not linear in $\{c_j\}$ even though $f^*$ is (since the star operator is not linear).
In practice, we can initialize $\{c_j\}$ by training without $R(f)$ 
(\ie, setting $R(f)$ to 0).
Despite the nonconvexity,
we have achieved a new solution technique for a broad class of inverse problems,
where the regularizer is a semi-norm.

\section{Convex Representation Learning by Euclidean Embedding}
\label{sec:euc_embed}

Interestingly, our framework---which so far only learns a predictive model---can be directly extended to learn structured \textit{representations} in a \emph{convex} fashion.
In representation learning, one identifies an ``object'' for each example $x$,
which, in our case, can be a function in $\Fcal$ or a vector in Euclidean space.
Such a representation is supposed to have incorporated the prior invariances in $R$,
and can be directly used for other (new) tasks such as supervised learning without further regularizing by $R$.
This is different from the RRM in Section \ref{sec:RRM},
which, although still enjoys the representer theorem in the applications we consider,
only seeks a discriminant function $f$ without providing a new representation for each example.

Our approach to convex representation learning is based on Euclidean embeddings (a.k.a. finite approximation or linearization) of the kernel representers in a s.i.p. space,
which is analogous to the use of RKHS in extracting useful features.
However, different from RKHS, 
$G_x$ and $G^*_x$ play different roles in a \sip\ space,
hence require \emph{different} embeddings in $\RR^d$.
For any $f \in \Bcal$ and $g^* \in \Bcal^*$,
we will seek their Euclidean embeddings $\iota(f)$ and $\iota^*(g^*)$, respectively.
Note $\iota^*$ is just a notation, not to be interpreted as ``the adjoint of $\iota$.''


We start by identifying the properties that a reasonable Euclidean embedding should satisfy intuitively.
Motivated by the bilinearity of $\dual{\cdot}{\cdot}_\Bcal$,
it is natural to require
\begin{align}
\label{eq:target_FA}
    \dual{f}{g^*}_\Bcal \approx \inner{\iota(f)}{\iota^*(g^*)}, \quad \forall f \in \Bcal, g^* \in \Bcal^*,
\end{align}
where $\inner{\cdot}{\cdot}$ stands for Euclidean inner product.
As $\dual{\cdot}{\cdot}_\Bcal$ is bilinear,
$\iota$ and $\iota^*$ should be linear on $\Bcal$ and $\Bcal^*$ respectively.
Also note $\iota^*((f+g)^*) \neq \iota^*(f^*) + \iota^*(g^*)$ in general.

Similar to the linearization of RKHS kernels,
we can apply invertible transformations to $\iota$ and $\iota^*$.
For example, doubling $\iota$ while halving $\iota^*$ makes no difference.
We will just choose one representation out of them.
It is also noteworthy that in general, 
$\nbr{\iota(f)}$ (Euclidean norm) approximates $\nbr{f}_\Hcal$ instead of
$\nbr{f}_\Bcal$.
\eqref{eq:target_FA} is the only property that our Euclidean embedding needs to satisfy.

We start by embedding the unit ball $\Bsf \! := \{f \in \Fcal : \nbr{f}_\Bcal \le 1\}$.
Characterizing $R$ by support functions as in Proposition~ \ref{prop:reg_R_supfun},
a natural Euclidean approximation of $\nbr{\cdot}_\Bcal$ is
%
\begin{align}
\label{eq:def_B1_FA}
	\nbr{v}_\Bcaltil^2 &:= \nbr{v}^2 
	+ \max\nolimits_{g \in S} \inner{v}{\gtil}^2, \quad \forall\ v \in \RR^d, 
\end{align}
where $\gtil$ is the Euclidean embedding of $g$ in the original RKHS,
designed to satisfy that $\langle \ftil, \gtil \rangle \approx \inner{f}{g}_\Hcal$ for all $f, g \in \Hcal$ (or a subset of interest).
Commonly used embeddings include Fourier \citep{RahRec08}, hash \citep{ShiPetDroetal09}, Nystr\"om \citep{WilSee00b}, etc.
For example, given landmarks $\{z_i\}_{i=1}^n$ sampled from $\Xcal$,
the Nystr\"om approximation for a function $f \in \Hcal$ is
%
\begin{align}
\label{eq:Nystrom}
    \ftil &= K_z^{-1/2} (f(z_1), \ldots, f(z_n))^\top \\
    \where K_z &:= [k(z_i, z_j)]_{i,j} \in \RR^{n \times n}.
\end{align}
Naturally, the dual norm of $\nbr{\cdot}_\Bcaltil$ is
%
\begin{align}
    \nbr{u}_{\tilde{\Bcal}^*} &:= \max\nolimits_{v : \nbr{v}_\Bcaltil \le 1} \inner{u}{v}, \quad \forall\ u \in \RR^d. 
\end{align}
Clearly the unit ball of $\nbr{\cdot}_\Bcaltil$ and $\nbr{\cdot}_{\Bcaltil^*}$ are also symmetric,
and we denote them as $\Bsftil$ and $\Bsftil^*$, respectively.


\begin{figure}
\centering
\begin{tikzcd}[column sep=huge,row sep=huge]
\Bcal \arrow[r,scale=2,"\iota"] \arrow[d,shift left=.75ex,"j"] &
  \Bcaltil \arrow[d,shift left=.75ex,"\jtil"] \\
\Bcal^* \arrow[r,"\iota^*"] \arrow[u,,shift left=.75ex, "j^{-1}"] 
& \Bcaltil^* \arrow[u,shift left=.75ex,"\jtil^{-1}"]
\end{tikzcd}
\caption{The commutative diagram for our embeddings.}
    \label{fig:diagram}
\end{figure}
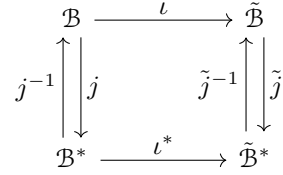

As shown in \Cref{fig:diagram}, we have the following commutative diagram. Let $j: \Bcal \to \Bcal^*$ be the star operator and $j^{-1}$ its inverse, and similarly for $\jtil: \Bcaltil \to \Bcaltil^*$ and its inverse $\jtil^{-1}$. Then, it is natural to require
%
\begin{align}
\label{eq:def_primal_FA}
    \iota = \jtil^{-1}\circ \iota^* \circ j,
\end{align}
where $\jtil^{-1}$ can be computed for any $u := \iota^*(f^*)$ via a Euclidean counterpart of \Cref{them:fx_unique}:
%
\begin{align}
    \jtil^{-1}(u) := \nbr{u}_{\Bcaltil^*} \cdot  
    \argmax\nolimits_{v \in \Bsftil} \inner{v}{u}.
\end{align}
The argmax is unique because $\nbr{\cdot}_\Bcaltil$ is strictly convex.

At last, how can we get $\iota^*(f^*)$ in the first place?
We start from the simpler case where $f^*$ has a kernel expansion as in 
\eqref{eq:representer_thm}.%
\footnote{We stress that although the kernel expansion \eqref{eq:representer_thm} is leveraged to \textit{motivate} the design of $\iota^*$, 
the underpinning foundation is that the span of $\{G_x^*: x \in \Xcal\}$ is dense in $\Bcal^*$ (Theorem \ref{thm:iotastar_linear}).
The representer theorem \citep[][Theorem 2]{ZhaXuZha09},
which showed that the solution to \eqref{eq:obj_rrm_rkbs} must be in the form of \eqref{eq:representer_thm},
is \textit{not} relevant to our construction.}
Here, by the linearity of $\iota^*$,
it will suffice to compute $\iota^*(G_x^*)$.
By Theorem \ref{them:fx_unique}, 
%
\begin{align*}
   G_x = G^{\circ}_x(x) G^{\circ}_x, \where G^{\circ}_x := \argmax\nolimits_{h \in \Bsf} h(x)
\end{align*} 
is uniquely attained.
Denoting $k_y := k(\cdot, y)$, it follows
%
\begin{align*}
    \dual{G_x}{G_y^*}_\Bcal \! \overset{\text{by }\eqref{eq:eval_fx_rkbs}}{=} \! 
    G(x, y) = 
    \inner{G_x}{k_y}_\Hcal \! 
    = \inner{G^{\circ}_x(x) G^{\circ}_x}{k_y}_\Hcal.
\end{align*}
So by comparing with \eqref{eq:target_FA},
it is natural to introduce
%
\begin{align}
    \iota^*(G_y^*) &:= \ktil_y, \\
\label{eq:def_primal_FA_rep}
    \iota(G_x) &:= G^{\circ}_x(x) \Gtil^{\circ}_x \approx \inner{\Gtil^{\circ}_x}{\ktil_x} \Gtil^{\circ}_x, \\
    \where \Gtil^{\circ}_x &:= \argmax\nolimits_{v \in \tilde{\Bsf}} \inner{v}{\ktil_x}.
\end{align}
The last optimization is \emph{convex} and can be solved very efficiently because, thanks to the positive homogeneity of $R$,
it is equivalent to
%
\begin{align}
\label{eq:po_equivalent}
	\min_{v} \big\{\nbr{v}^2 + \max\nolimits_{g \in S} \inner{v}{\gtil}^2 \big\} \quad s.t. \quad  v^\top \ktil_x = 1.
\end{align}
Detailed derivation and proof are relegated to Appendix \ref{app:solve_po}.
To solve \eqref{eq:po_equivalent},
LBFGS with projection to a hyperplane (which has a trivial closed-form solution) turned out to be very efficient in our experiment.
%
%
Overall, the construction of $\iota(f)$ and $\iota^*(f^*)$ for $f^*$ from \eqref{eq:representer_thm} proceeds as follows:
\begin{enumerate}[topsep=1pt,itemsep=3pt]
    \item Define $\iota^*(G_x^*) = \ktil_x$;
    \item Define $\iota^*(f^*) = \sum_i \alpha_i \ktil_{x_i}$ for $f^* = \sum_i \alpha_i G^*_{x_i}$;
    \item Define $\iota(f)$ based on $\iota^*(f^*)$ by using \eqref{eq:def_primal_FA}.
\end{enumerate}

In the next subsection, we will show that these definitions are sound,
and both $\iota$ and $\iota^*$ are linear.
However, the procedure may still be inconvenient in computation,
because $f$ needs to be first dualized to $f^*$,
which in turn needs to be expanded into the form of \eqref{eq:representer_thm}.
Fortunately, our representation learning only needs to compute the embedding of $G_x$,
bypassing all these computational challenges.


\subsection{Analysis of Euclidean Embeddings}
\label{sec:analysis_embedding}

The previous derivations are based on the necessary conditions for \eqref{eq:target_FA} to hold.
We now show that $\iota$ and $\iota^*$ are well-defined,
and are linear.
To start with,
denote the base Euclidean embedding on $\Hcal$ by $T: \Hcal \to \RR^d$,
where $T(f) = \ftil$.
Then by assumption, $T$ is linear and $\ktil_x = T(k(\cdot, x))$.

\begin{theorem}	
	\label{thm:iotastar_linear}
	$\iota^*(f^*)$ is well defined for all $f^* \in \Bcal^*$,
	and $\iota^*: \Bcal^* \to \RR^d$ is linear.
	That is, 
	\begin{enumerate}[label=\alph*),topsep=0pt,itemsep=3pt]
		\item If $f^* = \sum_i \alpha_i G_{x_i}^* = \sum_j \beta_j G_{z_j}^*$ are two different expansions of $f^*$,
		then $\sum_i \alpha_i \ktil_{x_i} = \sum_j \beta_j \ktil_{z_j}$.
		\item The linear span of $\{G_x^* : x \in \Xcal\}$ is dense in $\Bcal^*$.
		So extending the above to the whole $\Bcal^*$ is straightforward thanks to the linearity of $T$.
	\end{enumerate}
\end{theorem}

We next analyze the linearity of $\iota$.
To start with, we
make two assumptions on the Euclidean embedding of $\Hcal$.

\begin{assumption}[surjectivity]
	\label{assum:surjective}
	For all $v \in \RR^d$, there exists a $g_v \in \Hcal$ such that $\gtil_v = v$.
\end{assumption}
Assumption \ref{assum:surjective} does not cost any generality,
because it is satisfied whenever the $d$ coordinates of the embedding are linearly independent.
Otherwise, this can still be enforced easily by projecting to an orthonormal basis of $\{\gtil: g \! \in \Hcal\}$.
 
\begin{assumption}[lossless]
	\label{assum:lossless}	
	 $\inner{\ftil}{\gtil} = \inner{f}{g}_\Hcal$ for all $f, g \in \Hcal$.
	This is possible when, \eg, $\Hcal$ is finite dimensional.
\end{assumption}
 
\begin{theorem}	
	\label{thm:iot_linear}
	$\iota \! :\! \Bcal \! \to \! \RR^d$ is linear under Assumptions \ref{assum:surjective} \& \ref{assum:lossless}.
\end{theorem}
 
Although Theorems \ref{thm:iotastar_linear} and \ref{thm:iot_linear} appear intuitive,
the proof for the latter is rather nontrivial and is deferred to Appendix \ref{app:proof}.
Some lemmas there under Assumptions \ref{assum:surjective} and \ref{assum:lossless} may be of interest too,
hence highlighted here.
\begin{enumerate}[topsep=0pt,itemsep=0pt]
    \item $\inner{\iota(f)}{\iota^*(g^*)} = \dual{f}{g^*}$, 
	$\forall\ f \in \Bcal, g^* \in \Bcal^*$.
    \item $\nbr{g}_{\Bcal} = \nbr{g^*}_{\Bcal^*} = \nbr{\iota^*(g^*)}_{\Bcaltil^*}
	= \nbr{\iota(g)}_{\Bcaltil}$, $\forall\ g \in \Bcal$.
	\item $\Bsftil = \iota(\Bsf) := \{\iota(f) : \nbr{f}_\Bcal \le 1\}$.
	\item $\Bsftil^* = \iota^*(\Bsf^*) := \{\iota^*(g^*) : \nbr{g^*}_{\Bcal^*} \le 1\}$.
	\item $\max_{v \in \Bsftil} \inner{v}{\iota^*(g^*)}
	= 
	\max_{f \in \Bsf} \dual{f}{g^*},
	\ \ \forall g^* \in \Bcal^*$.
\end{enumerate}

\subsection{Analysis under Inexact Euclidean Embedding}

When Assumption \ref{assum:lossless} is unavailable,
Theorem \ref{thm:iotastar_linear} still holds,
but the linearity of $\iota$ has to be relaxed to an approximate sense.
To analyze it,
we first rigorously quantify the inexactness of the Euclidean embedding $T$.
Consider a subspace based embedding, such as Nystr\"om approximation.
Here $T$ satisfies that there exists a countable set of orthonormal bases $\{e_i\}_{i=1}^\infty$ of $\Hcal$,
such that 
\begin{enumerate}[topsep=0pt,itemsep=0pt]
    \item $T e_k = 0$ for all $k > d$,
    \item $\inner{Tf}{Tg} = \inner{f}{g}_\Hcal$, \ \ $\forall f, g \in V \! := \! \text{span} \{e_1, \ldots, e_d\}$.
\end{enumerate}
Clearly the Nystr\"om approximation in \eqref{eq:Nystrom} satisfies these conditions, 
where $d = n$,
and $\{e_1, \ldots, e_d\}$ is any orthornormal basis of $\{k_{z_1}, \ldots, k_{z_d}\}$ 
(assuming $d$ is no more than the dimensionality of $\Hcal$).
\begin{definition}
\label{def:eps_approx}
    $f \in \Hcal$ is called $\epsilon$-approximable by $T$ if 
    \begin{align}
        \nbr{f - \sum\nolimits_{i=1}^d \inner{f}{e_i}_\Hcal e_i}_\Hcal \le \epsilon.
    \end{align}
    In other words, the component of $f$ in $V^\perp$ is at most $\epsilon$.
\end{definition}

\begin{theorem}[The proof is in Appendix \ref{sec:app_inexact}]
\label{thm:gap_inexact}
    Let $f, g \in \Fcal$  and $\alpha \in \RR$.
    Then $\iota(\alpha f_1) = \alpha \iota(f_1)$.
    If $f$, $g$, and all elements in $S$ are $\epsilon$-approximable by $T$,
    then 
    \begin{align}
    \label{eq:gap_bilinear_inexact}
        \abr{\inner{\iota(f)}{\iota^*(g^*)} - \dual{f}{g^*}} &= O(\sqrt{\epsilon}) \\
    \label{eq:gap_linear_inexact}
        \nbr{\iota(f+ g) - \iota(f) - \iota(g)} &= O(\sqrt{\epsilon}).
    \end{align}
\end{theorem}

To summarize, the primal embedding $\iota(G_x)$ as defined in \eqref{eq:def_primal_FA_rep} provides a new feature representation that incorporates structures in the data.
Based on it, a simple linear model can be trained to achieve the desired regularities in prediction.
We now demonstrate its flexibility and effectiveness on two example applications.

\section{Application 1: Mixup}
\label{sec:mixup}

Mixup is a data augmentation technique \citep{ZhaCisDauLop18}, 
where a pair of training examples $x_i$ and $x_j$ are randomly selected, 
and their convex interpolation is postulated to yield the same interpolation of output labels.
In particular, when $y_i \in \{0,1\}^m$ is the one-hot vector encoding the class that $x_i$ belongs to,
the loss for the pair is 
\begin{align}
\label{eq:mixup_vanilla}
	\EE_\lambda [\ell(f(\underbrace{\lambda x_i + (1-\lambda) x_j}_{=: \ \xtil_\lambda}), \
	\underbrace{\lambda y_i + (1-\lambda) y_j}_{=: \ \ytil_\lambda})].
\end{align}
Existing literature relies on stochastic optimization,
with a probability pre-specified on $\lambda$.
This is somewhat artificial.
Changing expectation to maximization appears more appealing,
but no longer amenable to stochastic optimization.

To address this issue and to learn representations that incorporate mixup prior while also accommodating classification with multiclass or even structured output,
we resort to a joint kernel $k((x, y), (x',y'))$,
whose simplest form is decomposed as $k^x(x,x') k^y(y,y')$.
Here $k^x$ and $k^y$ are separate kernels on input and output respectively.
Now a function $f(x,y)$ learned from the corresponding RKHS quantifies the ``compatibility'' between $x$ and $y$,
and the prediction can be made by $\argmax_y f(x,y)$.
In this setting, 
the $R(f)$ for mixup regularization can leverage the $\ell_p$ norm of $g_{ij}(\lambda) := \frac{\partial}{\partial \lambda} f(\xtil_\lambda, \ytil_\lambda)$ over $\lambda \in [0,1]$,
effectively accounting for an infinite number of invariances.


\begin{theorem}
\label{thm:mixup_admissible}
	$R_{ij}(f) := \nbr{g_{ij}(\lambda)}_p$ satisfies Assumption \ref{assump:Reg} for all $p \in (1, \infty)$.
	The proof is in Appendix \ref{app:proof}.
\end{theorem}


Clearly taking \emph{expectation} or \emph{maximization} over all pairs of $n$ training examples still satisfies Assumption \ref{assump:Reg}.
%
In our experiment, we will use the $\ell_\infty$ norm,
which despite not being covered by \Cref{thm:mixup_admissible},
is directly amenable to the embedding algorithm.
More specifically, 
for each pair $(x, y)$ we need to embed $k((\cdot,\cdot), (x,y))$ as a $d \times m$ matrix.
This is different from the conventional setting where each example $x$ employs one feature representation shared for all classes;
here the representation changes for different classes $y$.
To this end, we need to first embed each invariance 
$g_{ij}(\lambda)$ by
\begin{align*}
	Z^{ij}_\lambda := \smallfrac{\partial}{\partial \lambda} \big(\ktil_{\xhat_\lambda} \ytil_\lambda^\top\big)
	= \Big(\smallfrac{\partial}{\partial \lambda} \ktil_{\xhat_\lambda}\Big) \ytil_\lambda^\top
	+ \ktil_{\xhat_\lambda} (y_i - y_j)^\top.
\end{align*}
%
Letting $\inner{A}{B} = \tr(A^\top B)$ and $\nbr{V}^2_{\Fsf} \! = \! \inner{V}{V}$,
the Euclidean embedding $\iota(G_{x,y})$ can be derived by solving \eqref{eq:po_equivalent}:
\begin{align}
\label{eq:obj_mixup_kernel_embed}
	\min_{V \in \RR^{d \times m}} &\bigg\{\alpha\nbr{V}_{\Fsf}^2 + 
	\frac{1}{n^2} \sum\limits_{ij} \max_{\lambda \in [0,1]} \inner{V}{Z^{ij}_\lambda}^2 \bigg\} \\
	s.t. \ & \inner{V}{\ktil_x y^\top} = 1.
\end{align}
Although the maximization over $\lambda$ in \eqref{eq:obj_mixup_kernel_embed} is not concave, 
it is 1-D and a grid style search can solve it globally with 
$O(\frac{1}{\epsilon})$ 
complexity.
In practice, a local solver like L-BFGS almost always found its global optimum in 10 iterations.

%

\section{Application 2: Embedding Inference for Structured Multilabel Prediction}
\label{sec:multilabel}

In output space, there is often prior knowledge about pairwise or multi-way relationships between labels/classes.
For example, if an image represents a cat, then it must represent an animal, but not a dog (assuming there is at most one object in an image).
Such logic relationships of implication and exclusion can be highly useful priors for learning \citep{MirRavDinSch15,Dengetal12}.
One way to leverage it is to perform inference at test time so that the predicted multilabel conforms to these logic.
However, this can be computation intensive at test time,
and it will be ideal if the predictor has already accounted for these logic,
and at test time, one just needs to make binary decisions (relevant/irrelevant) for each individual category separately.
We aim to achieve this by learning a representation that embeds this structured prior.

To this end, it is natural to employ the joint kernel framework.
We model the implication relationship of $y_1 \to y_2$ by enforcing 
$f(x, y_2) \ge  f(x, y_1)$,
which translates to a penalty on the amount by which $f(x, y_1)$ is above $f(x, y_2)$
%
\begin{align}
	[f(x, y_1) - f(x, y_2)]_+,
	\ \text{ where } [z]_+ = \max \{0, z\}.
\end{align}
%
To model the mutual exclusion relationship of $y_1 \leftrightsquigarrow y_2$,
intuitively we can encourage that $f(x, y_1) + f(x, y_2)\le 0$,
\ie, a higher likelihood of being a cat demotes the likelihood of being a dog.
It also allows both $y_1$ and $y_2$ to be irrelevant,
\ie, both $f(x, y_1)$ and $f(x, y_2)$ are negative.
This amounts to another sublinear penalty on $f$:
$	[f(x, y_1) + f(x, y_2)]_+$.
To summarize, letting $\ptil$ be the empirical distribution,
we can define $R(f)$ by
%
\begin{align}
\label{eq:invstrucml}
	R(f)^2 := \expunder{x \sim \ptil} &\Big[\max_{y_1 \to y_2} [f(x, y_1) - f(x, y_2)]_+^2  \\
	& \quad + \max_{y_1 \leftrightsquigarrow y_2} [f(x, y_1) + f(x, y_2)]_+^2 \Big].
\end{align}
%
It is noteworthy that although $R(f)$ is positively homogeneous and convex (hence sublinear),
it is no longer absolutely homogeneous and therefore not satisfying Assumption \ref{assump:Reg}.
However, the embedding algorithm is still applicable without change.
It will be interesting to study the presence of kernel function $G$ in spaces ``normed'' by sublinear functions.
We leave it for future work.

\input{experiment}

\section{Conclusions and Future Work}

In this paper, we introduced a new framework of representation learning where an RKHS is turned into a semi-inner-product space via a semi-norm regularizer,
broadening the applicability of kernel warping to \textit{generalized} invariances,
\ie, relationships that hold irrespective of certain changes in data.
For example, 
the mixup regularizer enforces smooth variation irrespective of the interpolation parameter $\lambda$,
and the structured multilabel regularizer enforces logic relationships between labels regardless of input features.
Neither of them can be modeled convexly by conventional methods in transformation invariance,
and the framework can also be directly applied to non-parametric transformations \citep{PalKanAraSav17}.
An efficient Euclidean embedding algorithm was designed and its theoretical properties are analyzed.
Favorable experimental results were demonstrated for the above two applications.

This new framework has considerable potential of being applied to other invariances and learning scenarios.
For example, it can be directly used in maximum mean discrepancy  and the Hilbert–Schmidt independence criterion,
providing efficient algorithms that complement the mathematical analysis in \citet{FukLanSri11}.
It can also be applied to convex deep neural networks \citep{Ganapathiraman18,GanZhaYuWen16},
which convexify multi-layer neural networks through kernel matrices of the hidden layer outputs.

Other examples of generalized invariance include \emph{convex} learning of: 
a) node representations in large networks that are robust to topological perturbations \citep{ZugAkbGun18}. 
The exponential number of perturbation necessitates max instead of sum; 
b) equivariance based on the largest deviation under swapped transformations over the input domain \citep{RavSchPoc17}; 
and
c) co-embedding multiway relations that preserve co-occurrence and affinity between groups \citep{MirWhiGyoSch15}.

\section*{Acknowledgements}

We thank the reviewers for their constructive comments.
This work is supported by Google Cloud and NSF grant RI:1910146. YY thanks NSERC and the Canada CIFAR AI Chairs program for funding support.

{
\bibliographystyle{abbrvnat}
\bibliography{bibfile}
}
\clearpage


\appendix
\begin{center}
	\Large \textbf{Appendix}
\end{center}

The appendix has two major parts: 
proof for all the theorems and more detailed experiments 
(Appendix \ref{app:experiment}).
\section{Proofs}
\label{app:proof}
{\bf Proposition 1.}
	$R(f)$ satisfies Assumption \ref{assump:Reg} if, and only if,
  $R(f) = \sup_{g \in S} \inner{f}{g}_\Hcal$,
    where $S \subseteq \Hcal$ is bounded in the RKHS norm
    and is symmetric ($g \in S \Leftrightarrow -g \in S$).    
    
Recall

{\bf Assumption 1.}    
    We assume that $R: \Fcal \to \RR$ is a  semi-norm.
Equivalently, $R: \Fcal \to \RR$ is convex and $R(\alpha f) = \abr{\alpha} R(f)$ for all $f \in \Fcal$ and $\alpha \in \RR$ (absolute homogeneity).
Furthermore, we assume $R$ is closed (i.e., lower semicontinuous) \wrt the topology in $\Hcal$.

Proposition \ref{prop:reg_R_supfun} (in a much more general form), to our best knowledge, is due to \citet{Hormander54}. We give a ``modern'' proof below for the sake of completeness.

\begin{proof}[\underline{Proof for Proposition \ref{prop:reg_R_supfun}}]
$\phantom{.}$ \newline	
    The ``if'' part: convexity and absolute homogeneity are trivial.
    To show the lower semicontinuity, we just need to show the epigraph is closed. Let $(f_n, t_n)$ be a convergent sequence in the epigraph of $R$,
    and the limit is $(f, t)$.
    Then $\inner{f_n}{g}_\Hcal \le t_n$ for all $n$ and $g \in S$.
    Tending $n$ to infinty, we get $\inner{f}{g}_\Hcal \le t$.
    Take supremum over $g$ on the left-hand side, 
    and we obtain $R(f) \le t$,
    i.e., $(f, t)$ is in the epigraph of $R$.
    
    The ``only if'' part: 
    A sublinear function $R$ vanishing at the origin is a support function if, and only if, it is closed. Indeed, if $R$ is closed, then its conjugate function
    \begin{align}
        \lambda R^*(f^*) &= \lambda \rbr{\sup_{f} \inner{f}{f^*}_{\Hcal} - R(f)} \\
        &= \sup_f \inner{\lambda f}{f^*}_{\Hcal} - R(\lambda f) \\
        &= R^*(f^*), 
    \end{align}
    is scaling invariant for any positive $\lambda$, i.e., $R^*$ is an indicator function. Conjugating again we have $R = (R^*)^*$ is a support function. 
    So, $R$ is the support function of
    \[
    S = \dom(R^*) = \{g : \inner{f}{g}_\Hcal \le R(f) \text{ for all } f \in \Hcal\},
    \]
    which is obviously closed.
    $S$ is also symmetric, because the symmetry of $R$ implies the same for its conjugate function $R^*$, hence its domain $S$.

    To see $S$ is bounded, assume to the contrary we have $\lambda_n g_n \in S$ with $\|g_n\|_{\Hcal} = 1$ and $\lambda_n \to \infty$. Since $R$ is finite-valued and closed, it is continuous, see \citep[e.g.][Proposition 4.1.5]{BorwenVanderwerff10}. Thus, for any $\delta > 0$ there exists some $\epsilon > 0$ such that $\|f\|_{\Hcal} \leq \epsilon \implies R(f) \leq \delta$. Choose $f = \epsilon g_n$ in the definition of $S$ above we have:
    \begin{align}
    \epsilon \lambda_n = \inner{ \epsilon g_n}{\lambda_n g_n}_{\Hcal} \leq R(\epsilon g_n) \leq \delta,
    \end{align}
    which is impossible as $\lambda_n \to \infty$.
    %
    %
\end{proof}

\begin{proof}[\underline{Proof of Theorem \ref{thm:iotastar_linear}}]
	$\phantom{.}$ \newline	
	a):
	since $\sum_i \alpha_i G_{x_i}^* \! = \! \sum_j \beta_j G_{z_j}^*$,
	it holds that
	\begin{align}
	\dual{h}{\sum_i \alpha_i G_{x_i}^*}
	=
	\dual{h}{\sum_j \beta_j G_{z_j}^*},
	\ \forall\ h \in \Fcal
	\end{align}
	which implies that
	\begin{align}
	\sum_i \alpha_i h(x_i) = \sum_j \beta_i h(z_j), \quad \forall\, h \in \Fcal.
	\end{align}
	Therefore 
	\begin{align}
	\sum_i \alpha_i k(x_i, \cdot) = \sum_j \beta_j k(z_j, \cdot).
	\end{align}
	Then apply the linear map $T$ on both sides, and we immediately get 
	$\sum_i \alpha_i \ktil_{x_i} = \sum_j \beta_j \ktil_{z_j}$.
	
	b): suppose otherwise that the completion of $\text{span}\{G^*_x:x \in \Xcal\}$ is not $\Bcal^*$.
	Then by the Hahn-Banach theorem, there exists a nonzero function $f \in \Bcal$ such that 
	$\dual{f}{G_x^*} = 0$ for all $x \in \Xcal$.
	By \eqref{eq:eval_fx_rkbs}, this means $f(x) = 0$ for all $x$.
	Since $\Bcal$ is a Banach space of functions on $\Xcal$, $f = 0$ in $\Bcal$.
	Contradiction.
	
	The linearity of $\iota^*$ follows directly from a) and b).
\end{proof}

To prove Theorem \ref{thm:iot_linear}, we first introduce five lemmas.
To start with, we set up the concept of \emph{polar operator} that will be used extensively in the proof:
\begin{align}
\label{eq:def_grad_dual_FA}
	\mathop{\text{PO}}\nolimits_{\Bsftil}(u) 
	:=
	\argmax_{v \in \Bsftil} \inner{v}{u}, \quad
	\forall\, u \in \RR^d.
\end{align}

Here the optimization is convex,
and the argmax is uniquely attained because $\Bsftil$ is strictly convex.
So $\nbr{\cdot}_{\tilde{\Bcal}^*}$ is differentiable at all $u$,
and the gradient is
\begin{align}
	\grad \nbr{u}_{\tilde{\Bcal}^*} 
	=
	\mathop{\text{PO}}\nolimits_{\Bsftil}(u).
\end{align}
\begin{lemma}
	\label{lem:norm_presever_dual}
	Under Assumptions \ref{assum:surjective} and \ref{assum:lossless},
	\begin{align}
	\label{eq:lemma_norm_presever_dual}
	\nbr{g}_{\Bcal} = \nbr{g^*}_{\Bcal^*} = \nbr{\iota^*(g^*)}_{\Bcaltil^*}
	= \nbr{\iota(g)}_{\Bcaltil}, \quad \forall\ g \in \Bcal.
	\end{align} 
\end{lemma}
\begin{proof}
	The first equality is trivial,
	and the third equality is by the definition of $\iota(g)$ in \eqref{eq:def_primal_FA}.
	To prove the second equality, 
	let us start by considering $g^* = \sum_i \alpha_i G_{x_i}^*$.
	Then
	\begin{align}
	\nbr{\iota^*(g^*)}_{\Bcaltil^*} 
	&= \max_{v \in \Bsftil} \inner{v}{\iota^*(g^*)} \\
	&= \max_{v \in \Bsftil}
	\sum_i \alpha_i \inner{v}{\ktil_{x_i}} \\
	\nbr{g^*}_{\Bcal^*} &= \max_{f \in \Bsf} \dual{f}{g^*}
	= \max_{f \in \Bsf} \sum_i \alpha_i f(x_i) \\
	&= \max_{f \in \Bsf} \sum_i \alpha_i \inner{f}{k(x_i,\cdot)}_\Hcal \\
	&= \max_{f \in \Bsf} \sum_i \alpha_i \inner{\ftil}{\ktil_{x_i}},
	\end{align}
	where the last equality is by Assumption \ref{assum:lossless}.
	So it suffices to show that $\Bsftil = \{\ftil: f \in \Bsf\}$.
	
	``$\supseteq$'' is trivial because for all $f \in \Bsf$, 
	by Assumption \ref{assum:lossless},
	\begin{align}
	\nbr{\ftil}^2 + \max\limits_{z \in S} \inner{\ztil}{\ftil}^2
	=
	\nbr{f}^2_\Hcal + \max\limits_{z \in S} \inner{z}{f}^2_\Hcal \le 1.\!
	\end{align}
	``$\subseteq$'': for any $v \in \Bsftil$,
	Assumption \ref{assum:surjective} asserts that there exists $h_v \in \Hcal$ such that $\htil_v = v$.
	Then by Assumption \ref{assum:lossless},
	\begin{align}
	\nbr{h_v}_\Hcal^2 + \max\limits_{z \in S} \inner{z}{h_v}^2_\Hcal
	=
	\nbr{v}^2 + \max\limits_{z \in S} \inner{\ztil}{v}^2 \le 1.
	\end{align}
	Since both $\nbr{\cdot}_{\Bcal^*}$ and $\nbr{\cdot}_{\Bcaltil^*}$ are continuous,
	applying the denseness result in part b) of Theorem \ref{thm:iotastar_linear} completes the proof of the second equality in \eqref{eq:lemma_norm_presever_dual}.
\end{proof}
\begin{lemma}
	\label{lem:exact_FA}
	Under Assumptions \ref{assum:surjective} and \ref{assum:lossless},
	\begin{align}
	\label{eq:exact_FA}
	\inner{\iota(f)}{\iota^*(g^*)} 
	= 
	\dual{f}{g^*}, \quad 
	\forall\ f \in \Bcal, g^* \in \Bcal^*.
	\end{align}
\end{lemma}
\begin{proof}
	\begin{align}
\label{eq:start_derive_FA}
    &\dual{f}{g^*} 
    \overset{\text{by } \eqref{eq:dual_sip_connection}}{=} 
    [g^*, f^*]_{\Bcal^*} \\
    &=\lim_{t \to 0} \frac{1}{2t} \rbr{\nbr{f^* + t g^*}^2_{\Bcal^*} - \nbr{f^*}^2_{\Bcal^*}} \text{ (by \citet{Giles67})}\\
    & =
    \lim_{t \to 0} \frac{1}{2t} \sbr{
    \nbr{\iota^*(f^*) + t \iota^*(g^*)}^2_{\tilde{\Bcal}^*}
    -
    \nbr{\iota^*(f^*)}^2_{\tilde{\Bcal}^*}},
\end{align}
where the last equality is by Lemma \ref{lem:norm_presever_dual} and Theorem \ref{thm:iotastar_linear}.
Now it follows from the polar operator as discussed above that
\begin{align}
    \dual{f}{g^*} 
    &= 
    \inner{\nbr{\iota^*(f^*)}_{\tilde{\Bcal}^*} 
    \cdot \mathop{\text{PO}}\nolimits_{\Bsftil}(\iota^*(f^*))}
    {\iota^*(g^*)} \\
    &= \inner{\iota(f)}{\iota^*(g^*)}. \tag*{\qedhere}
\end{align}

\end{proof}
\begin{lemma}
	\label{lem:iota_Btil_1}
	Under Assumptions \ref{assum:surjective} and \ref{assum:lossless},
	\begin{align}
	\Bsftil = \iota(\Bsf) := \{\iota(f) : \nbr{f}_\Bcal \le 1\}.
	\end{align}
\end{lemma}
\begin{proof}
	``$\text{LHS} \supseteq \text{RHS}$'': by Lemma \ref{lem:norm_presever_dual},
	it is obvious that $\nbr{f}_\Bcal \le 1$ implies $\nbr{\iota(f)}_\Bcaltil \le 1$.
	
	``$\text{LHS} \subseteq \text{RHS}$'': 
	we are to show that for all $v \in \Bsftil$,
	there must exist a $f_v \in \Bsf$ such that $v = \iota(f)$.
	If $v = 0$, then trivially set $f_v = 0$.
	In general, due to the polar operator definition \eqref{eq:def_grad_dual_FA},
	there must exist $u \in \RR^d$ such that 
	\begin{align}
	v / \nbr{v}_\Bcaltil = \text{PO}_{\Bsftil}(u).
	\end{align}
	We next reverse engineer a $q^* \in \Bcal^*$ so that $\iota^*(g^*) = u$.
	By Assumption \ref{assum:surjective}, there exists $h_u \in \Hcal$ such that $\htil_u = u$.
	Suppose $h_u = \sum_i \alpha_i k_{x_i}$.
	Then define $q^* = \sum_i \alpha_i G_{x_i}^*$,
	and we recover $u$ by 
	\begin{align}
	\iota^*(q^*) = \sum_i \alpha_i \ktil_i = \htil_u = u.
	\end{align}
	Apply Lemma \ref{lem:norm_presever_dual} and we obtain
	\begin{align}
	\label{eq:qnorm_equals_unorm}
	\nbr{q}_\Bcal = \nbr{\iota^*(q^*)}_{\Bcaltil^*}
	= \nbr{u}_{\Bcaltil^*}.
	\end{align}
	Now construct
	\begin{align}
	f_v = \frac{\nbr{v}_\Bcaltil}{\nbr{q}_\Bcal} \ q.
	\end{align}
	We now verify that $v = \iota(f_v)$.  By linearity of $\iota^*$,
	\begin{align}
	\iota^*(f_v^*) = \frac{\nbr{v}_\Bcaltil}{\nbr{q}_\Bcal} \ \iota^*(q^*)
	= \frac{\nbr{v}_\Bcaltil}{\nbr{q}_\Bcal} \ u.
	\end{align}
	So $\text{PO}_{\Bsftil} (\iota^*(f_v^*)) = v / \nbr{v}_\Bcaltil$ and plugging into \eqref{eq:def_primal_FA},
	\begin{align}
	\iota(f_v) &= \nbr{\iota^*(f_v^*)}_{\Bcaltil^*} \text{PO}_{\Bsftil} (\iota^*(f_v^*)) \\
	&= \frac{\nbr{v}_\Bcaltil}{\nbr{q}_\Bcal} \nbr{u}_{\Bcaltil^*} \frac{1}{\nbr{v}_\Bcaltil} v \\
	&= v. \quad \text{(by \eqref{eq:qnorm_equals_unorm})} \qedhere
	\end{align}
\end{proof}
\begin{lemma}
	\label{lem:iotastar_Btilstar}
	Under Assumptions \ref{assum:surjective} and \ref{assum:lossless},
	\begin{align}
	\Bsftil^* = \iota^*(\Bsf^*) := \{\iota^*(g^*) : \nbr{g^*}_{\Bcal^*} \le 1\}.
	\end{align}
\end{lemma}
\begin{proof}
	``$\text{LHS} \supseteq \text{RHS}$'': 
	By definition of dual norm, any $g^* \in \Bsf^*$ must satisfy
	\begin{align}
	\dual{f}{g^*} \le 1, \quad \forall\ f \in \Bsf.
	\end{align}
	Again, by the definition of dual norm, we obtain
	\begin{align}
	\nbr{\iota^*(g^*)}_{\Bcaltil^*} 
	&= 
	\sup_{v \in \Bsftil} \inner{v}{\iota^*(g^*)} \\
	&= 
	\sup_{f \in \Bsf} \inner{\iota(f)}{\iota^*(g^*)} \quad \text{(Lemma \ref{lem:iota_Btil_1})}\\
	&=
	\sup_{f \in \Bsf} \dual{f}{g^*}  \quad \text{(by Lemma \ref{lem:exact_FA})} \\
	&\le 1.
	\end{align}
	
	``$\text{LHS} \subseteq \text{RHS}$'': 
	Any $u \in \RR^d$ with $\nbr{u}_{\Bcaltil^*} = 1$ must satisfy
	\begin{align}
	\max_{v \in \Bsftil} \inner{u}{v} = 1.
	\end{align}
	Denote $v = \argmax_{v \in \Bsftil} \inner{u}{v}$ which must be uniquely attained.
	So $\nbr{v}_\Bcaltil = 1$.
	Then Lemma \ref{lem:iota_Btil_1} implies that
	there exists a $f \in \Bsf$ such that $\iota(f) = v$.
	By duality,
	\begin{align}
	\label{eq:u_is_maximizer}
	\max_{u \in \Bsftil^*}\inner{v}{u} = 1,
	\end{align}	
	and $u$ is the unique maximizer.
	Now note
	\begin{align}
	\inner{v}{\iota^*(f^*)} = \inner{\iota(f)}{\iota^*(f^*)}
	= \dual{f}{f^*} = 1,
	\end{align}
	where the last equality is derived from Lemma \ref{lem:norm_presever_dual} with
	\begin{align}
	\nbr{f}_\Bcal = \nbr{\iota(f)}_\Bcaltil = \nbr{v}_\Bcaltil = 1.
	\end{align}
	Note from Lemma \ref{lem:norm_presever_dual} that 
	$\nbr{\iota^*(f^*)}_{\Bcaltil^*} = \nbr{f}_\Bcal = 1$.
	So $\iota^*(f^*)$ is a maximizer in \eqref{eq:u_is_maximizer},
	and as a result, $u = \iota^*(f^*)$.
	
	If $\nbr{u}_{\Bcaltil^*} < 1$,
	then just construct $f$ as above for $u / \nbr{u}_{\Bcaltil^*}$,
	and then multiply it by $\nbr{u}_{\Bcaltil^*}$.
	The result will meet our need thanks to the linearity of $\iota^*$ from Theorem \ref{thm:iotastar_linear}.
\end{proof}
\begin{lemma}
	\label{lem:max_equal_FA}
	Under Assumptions \ref{assum:surjective} and \ref{assum:lossless},
	\begin{align}
	\label{eq:max_equal_FA}
	\max_{v \in \Bsftil} \inner{v}{\iota^*(g^*)}
	= 
	\max_{f \in \Bsf} \dual{f}{g^*},
	\ \ \forall g^* \in \Bcal^*.
	\end{align}
	Moreover, by Theorem \ref{them:fx_unique},
	the argmax of the RHS is uniquely attained at $f = g / \nbr{g}_{\Bcal}$,
	and the argmax of the LHS is uniquely attained at 
	$v = \iota(g) / \nbr{\iota(g)}_\Bcaltil$.
\end{lemma}
\begin{proof}
	LHS $\ge$ RHS: 
	Let $f^{opt}$ be an optimal solution to the RHS.
	Then by Lemma \ref{lem:iota_Btil_1},
	$\iota(f^{opt}) \in \Bsftil$,
	and so	
	\begin{align}
	\text{RHS} &= \dual{f^{opt}}{g^*} \\
	&= \inner{\iota(f^{opt})}{\iota^*(g^*)} \quad \text{(by Lemma \ref{lem:exact_FA})}\\
	&\le 	\max_{v \in \Bsftil} \inner{v}{\iota^*(g^*)} \\
	&= \text{LHS}.
	\end{align}
	
	LHS $\le$ RHS: let $v^{opt}$ be an optimal solution to the LHS.
	Then by Lemma \ref{lem:iota_Btil_1},
	there is $f_{v^{opt}} \in \Bsf$ such that $\iota(f_{v^{opt}}) = v^{opt}$.
	So
	\begin{align}
	\text{LHS} &= \inner{v^{opt}}{\iota^*(g^*)} \\
	&= \inner{\iota(f_{v^{opt}})}{\iota^*(g^*)}  \\
	&= \dual{f_{v^{opt}}}{g^*}  \quad \text{(by Lemma \ref{lem:exact_FA})} \\
	&\le \max_{f \in \Bsf} \dual{f}{g^*} \quad (\text{since } f_{v^{opt}} \in \Bsf)\\
	&= \text{RHS}. \tag*{\qedhere}
	\end{align}
\end{proof}

\begin{proof}[\underline{Proof of Theorem \ref{thm:iot_linear}}]
	Let $f \in \Bcal$ and $\alpha \in \RR$.
	Then $(\alpha f)^* = \alpha f^*$, and by \eqref{eq:def_primal_FA} and Theorem \ref{thm:iotastar_linear},
	\begin{align}
	\iota(\alpha f) &= \nbr{\iota^*(\alpha f^*)}_{\Bcaltil^*} \cdot \text{PO}_{\Bsftil}(\iota^*(\alpha f^*)) \\
	&= \abr{\alpha} \nbr{\iota^*(f^*)}_{\Bcaltil^*} \cdot
	\text{PO}_{\Bsftil}(\alpha \iota^*(f^*)).
	\end{align}
	By the symmetry of $\Bsftil$, 
	\begin{align}
	\iota(\alpha f) &= \abr{\alpha} \nbr{\iota^*(f^*)}_{\Bcaltil^*} \cdot \text{sign}(\alpha) \, \text{PO}_{\Bsftil}(\iota^*(f^*)) \\
	&= \alpha \, \iota(f).
	\end{align}
	Finally we show $\iota(f_1+f_2) = \iota(f_1) + \iota(f_2)$ for all $f_1, f_2 \in \Bcal$.
	Observe
	\begin{align}
	&\inner{\iota(f_1) + \iota(f_2)}{\iota^*((f_1+f_2)^*)} \\
	= \ &\inner{\iota(f_1)}{\iota^*((f_1+f_2)^*)} + \inner{\iota(f_2)}{\iota^*((f_1+f_2)^*)} \\
	= \ &\dual{f_1}{(f_1+f_2)^*} + 
	\dual{f_2}{(f_1+f_2)^*} \\
	= \ & \dual{f_1+f_2}{(f_1+f_2)^*}.
	\end{align}
	Therefore
	\begin{align}		
	\label{eq:f1_plus_f2_duality}
	\inner{v}{\iota^*((f_1+f_2)^*)} 	
	&= \dual{\frac{f_1+f_2}{\nbr{f_1+f_2}_\Bcal}}{ (f_1+f_2)^*}, \\
	\where v &= \frac{\iota(f_1) + \iota(f_2)}{\nbr{f_1+f_2}_\Bcal}.
	\end{align}
	
	We now show $\nbr{v}_\Bcaltil = 1$,
	which is equivalent to 
	\begin{align}
	\nbr{\iota(f_1) + \iota(f_2)}_\Bcaltil = \nbr{f_1+f_2}_\Bcal.
	\end{align}
	Indeed, this can be easily seen from
	\begin{align}
	\text{LHS} &= \sup_{u \in \Bsftil^*} \inner{\iota(f_1) + \iota(f_2)}{u} \\
	&= \sup_{g^* \in \Bsf^*} \inner{\iota(f_1) + \iota(f_2)}{\iota^*(g^*)} \ \ \text{(Lemma \ref{lem:iotastar_Btilstar})} \\
	&= \sup_{g^* \in \Bsf^*} \dual{f_1 + f_2}{g^*}
	\quad \text{(by Lemma \ref{lem:exact_FA})} \\
	&= \text{RHS}.
	\end{align}
	
	By Lemma	\ref{lem:max_equal_FA},
	\begin{align}
	\max_{v \in \Bsftil} \inner{v}{\iota^*((f_1+f_2)^*)}
	= 
	\max_{f \in \Bsf} \dual{f}{(f_1+f_2)^*}.
	\end{align}
	Since the right-hand side is optimized at $f = (f_1+f_2)/\nbr{f_1+f_2}_\Bcal$,
	we can see from \eqref{eq:f1_plus_f2_duality} and $\nbr{v}_\Bcaltil = 1$ that
	$v = \text{PO}_{\Bsftil} (\iota^*((f_1+f_2)^*))$.
	Finally by definition \eqref{eq:def_primal_FA}, we conclude
	\begin{align}
	\iota(f_1+f_2) &= \nbr{\iota^*((f_1+f_2)^*)}_{\Bcaltil^*} \cdot
	\text{PO}_{\Bsftil} (\iota^*((f_1+f_2)^*)) \\
	&= \nbr{f_1+f_2}_\Bcal \, v \quad \text{(by Lemma \ref{lem:norm_presever_dual})} \\
	&= \iota(f_1) + \iota(f_2). \tag*{\qedhere}
	\end{align}
\end{proof}

\begin{proof}[\underline{Proof of Theorem \ref{thm:mixup_admissible}}]
    We assume that the kernel $k$ is smooth and the function
    \[
    z_{ij}(\lambda) = \tfrac{\partial}{\partial \lambda} k((\xtil_\lambda, \ytil_\lambda), (\cdot, \cdot)).
    \]
    is in $L_p$ so that $R_{ij}$ is well-defined and finite-valued.
    
    Clearly, using the representer theorem we can rewrite 
    \begin{align}
    R_{ij}(f) = \| \inner{f}{z_{ij}(\lambda)}_{\Hcal} \|_p.
    \end{align} 
    Thus, $R_{ij}$ is the composition of the linear map $f \mapsto g(\lambda; f) := \inner{f}{z_{ij}(\lambda)}_{\Hcal}$ and the $L_p$ norm $g \mapsto \|g(\lambda)\|_p$. 
    It follows from the chain rule that $R_{ij}$ is convex, absolutely homogeneous, and G\^ateaux differentiable (recall that the $L_p$ norm is G\^ateaux differentiable for $p\in (1, \infty)$).
    %
    %
\end{proof}
    

\section{Analysis under Inexact Euclidean Embedding}
\label{sec:app_inexact}

We first rigorously quantify the inexactness in the Euclidean embedding $T$:
$\Hcal \to \RR^d$, where $Tf = \ftil$.
To this end, let us consider a subspace based embedding, such as Nystr\"om approximation.
Here let $T$ satisfy that there exists a countable set of orthonormal bases $\{e_i\}_{i=1}^\infty$ of $\Hcal$,
such that 
\begin{enumerate}
    \item $T e_k = 0$ for all $k > d$,
    \item $\inner{Tf}{Tg} = \inner{f}{g}_\Hcal$, \ \ $\forall f, g \in V \! := \! \text{span} \{e_1, \ldots, e_d\}$.
\end{enumerate}

Clearly the Nystr\"om approximation in \eqref{eq:Nystrom} satisfies these conditions, 
where $d = n$,
and $\{e_1, \ldots, e_d\}$ is any orthornormal basis of $\{k_{z_1}, \ldots, k_{z_d}\}$ 
(assuming $d$ is no more than the dimensionality of $\Hcal$).

As an immediate consequence,
$\{Te_1, \ldots, Te_d\}$ forms an orthonormal basis of $\RR^d$:
$\inner{Te_i}{Te_j} = \inner{e_i}{e_j}_\Hcal = \delta_{ij}$ for all $i, j \in [d]$.
Besides, $T$ is contractive because for all $f \in \Fcal$,
\begin{align}
    \nbr{Tf}^2 
    &= \nbr{\sum_{i=1}^d \inner{f}{e_i}_\Hcal Te_i}^2 \\
    &= \sum_{i=1}^d \inner{f}{e_i}^2_\Hcal 
    \le \nbr{f}^2_\Hcal.
\end{align}


By Definition \ref{def:eps_approx},
obviously $k_{z_i}$ is $0$-approximable under the Nystr\"om approximation.
If both $f$ and $g$ are $\epsilon$-approximable,
then $f+g$ must be $(2\epsilon)$-approximable.

\begin{lemma}
\label{lem:gap_uf_TuTf}
    Let $f \in \Hcal$ be $\epsilon$-approximable by $T$,
    then for all $u \in \Hcal$,
    \begin{align}
        \abr{\inner{u}{f}_\Hcal - \inner{Tu}{Tf}} \le \epsilon \nbr{u}_\Hcal.
    \end{align}
\end{lemma}
\begin{proof}
    Let $f = \sum_{i=1}^\infty \alpha_i e_i$ and $u = \sum_{i=1}^\infty \beta_i e_i$. Then
    \begin{align}
        &\abr{\inner{u}{f}_\Hcal - \inner{Tu}{Tf}} \\
        = &\abr{\sum_{i=1}^\infty \alpha_i \beta_i - 
        \inner{\sum_{i=1}^d \alpha_i T e_i}{\sum_{j=1}^d \beta_j T e_j}} \\
        = &\abr{\sum_{i=d+1}^\infty \alpha_i \beta_i} \\
        \le &\rbr{\sum_{i=d+1}^\infty \alpha_i^2}^{1/2} \rbr{\sum_{j=d+1}^\infty \beta_j^2}^{1/2} \\
        \le &\ \epsilon \nbr{u}_\Hcal. \tag*{\qedhere}
    \end{align}
\end{proof}

%
\begin{proof}[Proof of Theorem \ref{thm:gap_inexact}]
    We first prove \eqref{eq:gap_bilinear_inexact}.
    Note for any $u \in \Fcal$,
    \begin{align}
        \dual{u}{g^*} &= [u, g] \\
        &= \lim_{t \to 0} \frac{1}{2} \sbr{\nbr{tu+g}_\Bcal^2 - \nbr{g}_\Bcal^2} \\
        &= \inner{u}{g + \grad R^2(g)}_\Hcal.
    \end{align}
    The differentiability of $R^2$ is guaranteed by the G\^ateaux differentiability.
    Letting $g^* = \sum_i \alpha_i G_{v_i}^*$,
    it follows that 
    \begin{align}
        \dual{u}{g^*} &= \sum_i \alpha_i u(v_i) = \inner{u}{\sum_i \alpha_i k_{v_i}}_\Hcal.
    \end{align}
    So $\sum_i \alpha_i k_{v_i} = g + \grad R^2(g)$, and by the definition of $\iota^*$
    \begin{align}
    \label{eq:iotag_Tag}
        \iota^*(g^*) &= \sum_i \alpha_i T k_{v_i} = T a_g \\
        \where a_g &:= \sum_i \alpha_i k_{v_i} = g + \grad R^2(g).
    \end{align}
    Similarly,
        \begin{align}
        \iota^*(f^*) &= Ta_f, \where a_f := f + \grad R^2(f).
    \end{align}
    By assumption $\argmax_{h \in S} \inner{h}{g}_\Hcal$ is $\epsilon$-approximable,
    and hence $a_g$ is $O(\epsilon)$-approximable.
    Similarly, $a_f$ is also $O(\epsilon)$-approximable.
    
    Now let us consider
    \begin{align}
    \label{eq:def_v_proof}
        v^\circ &:= \argmax_{v \in \RR^d: \nbr{v}^2 + \sup_{h \in S} \inner{v}{Th}^2 \le 1} \inner{v}{Ta_f} \\
        u^\circ &:= \argmax_{u \in \Fcal: \nbr{u}_\Hcal^2 + \sup_{h \in S} \inner{u}{h}_\Hcal^2 \le 1} \inner{u}{a_f}_\Hcal.
    \end{align}
    By definition, $\iota(f) = v^\circ$.
    Also note that $u^\circ = f$ because  
    $\inner{u}{a_f}_\Hcal = \dual{u}{f^*}$ for all $u \in \Fcal$.
    We will then show that 
    \begin{align}
    \label{eq:gap_vu}
        \nbr{\iota(f) - T f} = \nbr{v^\circ - T u^\circ} = O(\sqrt{\epsilon}),
    \end{align}
    which allows us to derive that
    \begin{align}
        \dual{f}{g^*} &= \inner{f}{a_g}_\Hcal \\
        &= \inner{Tf}{Ta_g} + O(\epsilon) \ \ \text{ (by Lemma \ref{lem:gap_uf_TuTf})} \\
        &= \inner{Tu^\circ}{Ta_g} + O(\epsilon) \\
        &= \inner{v^\circ}{Ta_g} + O(\sqrt{\epsilon}) \quad \text{(by \eqref{eq:gap_vu})}\\
        &= \inner{\iota(f)}{\iota^*(g^*)} + O(\sqrt{\epsilon}). \quad \text{(by \eqref{eq:iotag_Tag})}
    \end{align}
    
    Finally, we prove \eqref{eq:gap_vu}.
    Denote
    \begin{align}
        w^\circ := \argmax_{w \in \Fcal: \nbr{w}_\Hcal^2 + \sup_{h \in S} \inner{Tw}{Th}^2 \le 1} \inner{w}{a_f}_\Hcal.
    \end{align}
    We will prove that $\nbr{v^\circ - T w^\circ} = O(\epsilon^2)$ and 
    $\nbr{u^\circ - w^\circ}_\Hcal = O(\sqrt{\epsilon})$.
    They will imply \eqref{eq:gap_vu} because by the contractivity of $T$,
    $\nbr{T(u^\circ - w^\circ)} \le \nbr{u^\circ - w^\circ}_\Hcal$.
    
    \textbf{Step 1}: $\nbr{v^\circ - T w^\circ} = O(\epsilon^2)$.
    Let $w = w_1 + w_2$ where $w_1 \in V$ and $w_2 \in V^\perp$.
    So $Tw = Tw_1$ and $\nbr{Tw} = \nbr{w_1}_\Hcal$.
    Similarly decompose $a_f$ as $a_1 + a_2$, where $a_1 = T a_f \in V$ and $a_2 \in V^\perp$.
    Now the optimization over $w$ becomes
    \begin{align}
    \label{eq:obj_w1w2}
        &\max_{w_1 \in V, w_2 \in V^\perp} \quad \inner{w_1}{a_1}_\Hcal + \inner{w_2}{a_2}_\Hcal\\
        &s.t. \quad \nbr{w_1}_\Hcal^2 + \nbr{w_2}_\Hcal^2 + \sup_{h \in S} \inner{Tw_1}{Th}^2 \le 1.
    \end{align}
    Let $\nbr{w_2}^2 = 1-\alpha$ where $\alpha \in [0,1]$.
    Then the optimal value of $\inner{w_2}{a_2}_\Hcal$ is 
    $\sqrt{1-\alpha} \nbr{a_2}_\Hcal$.
    Since $\inner{w_1}{a_1}_\Hcal = \inner{Tw_1}{Ta_1}$, 
    the optimization over $w_1$ can be written as
    \begin{align}
        \min_{w_1 \in V} &\inner{Tw_1}{Ta_1} \\
        s.t. \ & \nbr{Tw_1}^2 + \sup_{h \in S} \inner{Tw_1}{Th}^2 \le \alpha.
    \end{align}
    Change variable by $v = T w_1$.
    Then compare with the optimization of $v$ in \eqref{eq:def_v_proof},
    and we can see that $v^\circ = T w^\circ_1 / \sqrt{\alpha}$.
    Overall the optimal objective value of \eqref{eq:obj_w1w2} under $\nbr{w_2}^2 = 1-\alpha$ is
    $\sqrt{1-\alpha} \nbr{a_2}_\Hcal + \sqrt{\alpha} p$
    where $p$ is the optimal objective value of \eqref{eq:def_v_proof}.
    So the optimal $\alpha$ is $\frac{p^2}{p^2 + \nbr{a_2}_\Hcal^2}$,
    and hence
    \begin{align}
        \nbr{v^\circ - T w^\circ} &= \nbr{v^\circ - T w^\circ_1} = \nbr{v^\circ - \sqrt{\alpha} v^\circ}\\
        &= (1-\sqrt{\alpha}) \nbr{v^\circ} \le 1- \sqrt{\alpha}.
    \end{align}
    Since $a_f$ is $O(\epsilon)$-approximable,
    so $\nbr{a_2}_\Hcal = O(\epsilon)$ and 
    \begin{align}
        1-\sqrt{\alpha} = \frac{1 - \alpha}{1+\sqrt{\alpha}} 
        = O(\nbr{a_2}_\Hcal^2) = O(\epsilon^2).
    \end{align}

    \textbf{Step 2}: $\nbr{u^\circ - w^\circ}_\Hcal = O(\sqrt{\epsilon})$.
    Motivated by Theorem \ref{thm:flip_opt},
    we consider two equivalent problems:
    \begin{align}
        \uhat^\circ &= \argmax_{u \in \Fcal: \inner{u}{a_f}_\Hcal = 1}
        \cbr{\nbr{u}_\Hcal^2 + \sup_{h \in S} \inner{u}{h}_\Hcal^2} \\
        \what^\circ &= \argmax_{w \in \Fcal: \inner{w}{a_f}_\Hcal = 1}
        \cbr{\nbr{w}_\Hcal^2 + \sup_{h \in S} \inner{Tw}{Th}^2}.
    \end{align}
    Again we can decompose $u$ into $U := \text{span} \{a_f\}$ and its orthogonal space $U^\perp$.
    Since $\inner{u}{a_f}_\Hcal = 1$,
    the component of $u$ in $U$ must be $\abar_f := a_f / \nbr{a_f}_\Hcal^2$.
    So
    \begin{align}
    \label{eq:obj_u1circ}
        \uhat^\circ = \ \abar_f + 
        \argmax_{u^\perp \in U^\perp} \cbr{\nbr{u^\perp}_\Hcal^2 + 
        \sup_{h \in S} \inner{u^\perp + \abar_f}{h}_\Hcal^2 }.
    \end{align}
    Similarly,
     \begin{align}
     \label{eq:obj_w1circ}
        w^\circ = \ \abar_f + 
        \argmax_{w^\perp \in U^\perp} &\bigg\{\nbr{w^\perp}_\Hcal^2  \\ &+ 
        \sup_{h \in S} \inner{T(w^\perp + \abar_f)}{Th}_\Hcal^2 \bigg\}.
    \end{align}   
    We now compare the objective in the above two argmax forms.
    Since any $h \in S$ is $\epsilon$-approximable,
    so for any $x \in \Fcal$:
    \begin{align}
        \abr{\inner{x}{h}_\Hcal - \inner{Tx}{Th}_\Hcal} = O(\epsilon).
    \end{align}
    Therefore tying $u^\perp = w^\perp = x$,
    the objectives in the argmax of \eqref{eq:obj_u1circ} and \eqref{eq:obj_w1circ} differ by at most $O(\epsilon)$.
    Therefore their optimal objective values are different by at most $O(\epsilon)$.
    Since both objectives are (locally) strongly convex in $U^\perp$,
    the RKHS distance between the optimal $u^\perp$ and the optimal $w^\perp$ must be $O(\sqrt{\epsilon})$.
    As a result $\nbr{\uhat^\circ - \what^\circ}_\Hcal = O(\sqrt{\epsilon})$.
    
    Finally to see $\nbr{u^\circ - w^\circ}_\Hcal = O(\epsilon)$,
    just note that by Theorem \ref{thm:flip_opt},
    $u^\circ$ and $w^\circ$ simply renormalize $\uhat^\circ$ and $\what^\circ$ to the unit sphere of $\nbr{\cdot}_\Bcal$, respectively.
    So again $\nbr{u^\circ - w^\circ}_\Hcal = O(\sqrt{\epsilon})$.

%

    In the end, we prove \eqref{eq:gap_linear_inexact}.
    The proof of $\iota(\alpha f) = \alpha \iota(f)$ is exactly the same as that for Theorem \ref{thm:iotastar_linear}.
    To prove \eqref{eq:gap_linear_inexact}, 
    note that $f + g$ is $(2\epsilon)$-approximable.
    Therefore applying \eqref{eq:gap_vu} on $f$, $g$, $f + g$, we get
    \begin{align}
        \nbr{\iota(f) - Tf} &= O(\sqrt{\epsilon}), \\
        \nbr{\iota(fg) - Tg} &= O(\sqrt{\epsilon}), \\
        \nbr{\iota(f+g) - T(f+g)} &= O(\sqrt{\epsilon}).
    \end{align}
    Combining these three relations, we conclude \eqref{eq:gap_linear_inexact}.
\end{proof}

\section{Solving the Polar Operator}
\label{app:solve_po}

\begin{theorem}
\label{thm:flip_opt}
	Suppose $J$ is continuous and 
	$J(\alpha x) = \alpha^2 J(x) \ge 0$ for all $x$ and $\alpha \ge 0$.
	Then $x$ is an optimal solution to
	\begin{align}
	P: \quad \max_{x} a^\top x, \quad s.t. \quad J(x) \le 1,
	\end{align}
	if, and only if,
	$J(x) = 1$, $c := a^\top x > 0$, and 
	$
	\xhat := x / c
	$
	is an optimal solution to
	\begin{align}
	Q: \quad\min_x J(x), \quad s.t. \quad a^\top x = 1.
	\end{align}
\end{theorem}

\begin{proof}
	We first show the ''only if'' part.
	Since $J(0) = 0$ and $J$ is continuous,
	the optimal objective value of $P$ must be positive.
	Therefore $c > 0$.
	Also note the optimal $x$ for $P$ must satisfy $J(x) = 1$ because otherwise one can scale up $x$ to increase the objective value of $P$.
	To show $\xhat$ optimizes $Q$,
	suppose otherwise there exists $y$ such that
	\begin{align}
	a^\top y = 1, \quad J(y) < J(\xhat).
	\end{align}
	Then letting 
	\begin{align}
	z = J(y)^{-1/2} y,
	\end{align}
	we can verify that
	\begin{align}
	J(z) &= 1,\\
	a^\top z &= J(y)^{-1/2} a^\top y = J(y)^{-1/2} \\
	&> J(\xhat)^{-1/2} = c J(x)^{-1/2} = c = a^\top x.
	\end{align}
	So $z$ is a feasible solution for $P$,
	and is strictly better than $x$.
	Contradiction.
	
	We next show the ``if'' part:
	for any $x$,
	if $J(x) = 1$, $c := a^\top x > 0$, and $\xhat := x/c$ is an optimal solution to $Q$,
	then $x$ must optimize $P$.
	Suppose otherwise there exists $y$, such that $J(y) \le 1$ and $a^\top y > a^\top x > 0$.
	Then consider $z := y / a^\top y$.
	It is obviously feasible for $Q$,
	and
	\begin{align}
	J(z) &= (a^\top y)^{-2} J(y) < (a^\top x)^{-2} J(y) \\
	&\le (a^\top x)^{-2} J(x) = J(\xhat).
	\end{align}
	This contradicts with the optimality of $\xhat$ for $Q$.
\end{proof}

\paragraph{Projection to hyperplane}
To solve problem \eqref{eq:po_equivalent},
we use LBFGS with each step projected to the feasible domain, a hyperplane.
This requires solving, for given $c$ and $a$,
\begin{align}
\min_x \frac{1}{2} \nbr{x - c}^2, \quad s.t. \quad a^\top x = 1.
\end{align}
Write out its Lagrangian and apply strong duality thanks to convexity:
\begin{align}
&\min_x \max_\lambda \frac{1}{2} \nbr{x - c}^2 - \lambda(a^\top x - 1) \\
= \ & \max_\lambda \min_x \frac{1}{2} \nbr{x - c}^2 - \lambda(a^\top x - 1) \\
= \ &\max_\lambda \frac{1}{2} \lambda^2 \nbr{a}^2 - \lambda^2 \nbr{a}^2 - \lambda a^\top c + \lambda,
\end{align}
where $x = c + \lambda a$.
The last step has optimal 
\begin{align}
\lambda = (1 - a^\top c) / \nbr{a}^2.
\end{align}

\section{Gradient in Dual Coefficients}
\label{app:gradient_c}

We first consider the case where $S$ is a finite set,
and denote as $z_i$ the RKHS Nystr\"{o}m approximation of its $i$-th element.
When $f^*$ has the form of \eqref{eq:representer_thm},
we can compute $\iota(f)$ by using the Euclidean counterpart of Theorem \ref{them:fx_unique} as follows:
\begin{align}
    \argmax_u \ \ & u^\top \sum\nolimits_j c_j k_j \\
    s.t. \ &\nbr{u}^2 + (z_i^\top u)^2 \le 1, \quad \forall\ i,
\end{align}
where $k_j$ the the Nystr\"{o}m approximation of $k(x_j,\cdot)$.

Writing out the Lagrangian with dual variables $\lambda_i$:
\begin{align}
    u^\top \sum_j c_j k_j + \sum_i \lambda_i \rbr{\nbr{u}^2 + (z_i^\top u)^2-1},
\end{align}
we take derivative with respect to $u$:
\begin{align}
\label{eq:grad_c_grad_u}
    X^\top c + 2 \one^\top \lambda u + 2 Z \Lambda Z^\top u = 0.
\end{align}
where $X = (k_1, k_2, \ldots)$,
$Z = (z_1, z_2, \ldots)$,
$\lambda = (\lambda_1, \lambda_2, \ldots)$,
$\Lambda = \diag(\lambda_1, \lambda_2, \ldots)$ (diagonal matrix),
and $\one$ is a vector of all ones.
This will hold for $c + \Delta_c$, $\lambda + \Delta_\lambda$ and $u + \Delta_u$:
\begin{align}
    X^\top (c + \Delta_c) &+ 2 \one^\top (\lambda + \Delta_\lambda) (u + \Delta u) \\
    &+ 2 Z (\Lambda + \Delta_\Lambda) Z^\top (u + \Delta_u) = 0.
\end{align}
Subtract it by \eqref{eq:grad_c_grad_u},
we obtain
\begin{align}
\label{eq:grad_c_grad_u_2}
    X^\top \Delta_c &+ 2 (\one^\top \Delta_\lambda) u + 2 (\one^\top \lambda) \Delta_u \\
    &+ 2 Z \Delta_\Lambda Z^\top u + 2 Z \Lambda Z^\top \Delta_u = 0.
\end{align}

The complementary slackness writes
\begin{align}
\label{eq:grad_c_grad_lam}
    \lambda_i (\nbr{u}^2 + (z_i^\top u)^2 - 1) = 0.
\end{align}
This holds for $\lambda + \Delta_\lambda$ and $u + \Delta_u$:
\begin{align}
    (\lambda_i + \Delta_{\lambda_i}) (\nbr{u + \Delta_u}^2 + (z_i^\top u + z_i^\top \Delta_u)^2 - 1) = 0.
\end{align}
Subtract it by \eqref{eq:grad_c_grad_lam},
we obtain
\begin{align}
\label{eq:grad_c_grad_lam_2}
    \Delta_{\lambda_i} (\nbr{u}^2 + (z_i^\top u)^2 - 1)
    + 2 \lambda_i (u + (z_i^\top u) z_i)^\top \Delta_u = 0.
\end{align}
Putting together \eqref{eq:grad_c_grad_u_2} and \eqref{eq:grad_c_grad_lam_2},
we obtain
\begin{align}
\label{eq:grad_c_bridge}
    S 
    \begin{pmatrix}
    \Delta_u \\
    \Delta_\lambda
    \end{pmatrix}
    =
    \begin{pmatrix}
    -X^\top \Delta_c \\
    0
    \end{pmatrix},
\end{align}
where $S$ is
\begin{align}
    \begin{pmatrix}
    2 (\one^\top \lambda) I + 2 Z \Lambda Z^\top 
    & 2 u \one^\top + 2 Z \diag(Z^\top u) \\
    2 \Lambda (\one u^\top + \diag(Z^\top u) Z^\top) 
    & \diag(\nbr{u}^2 + (z_i^\top u)^2 - 1)
    \end{pmatrix}.
\end{align}
Therefore
\begin{align}
\label{eq:grad_c_dudc}
    \frac{\rmd u}{\rmd c} = 
    \begin{pmatrix} I & 0 \end{pmatrix}
    S^{-1} 
    \begin{pmatrix} -X^\top \\ 0 \end{pmatrix} .
\end{align}

Finally we investigate the case when $S$ is not finite.
In such a case, the elements $z$ in $S$ that attain
$\nbr{u}^2 + (z^\top u)^2 = 1$ for the optimal $u$ are still finite in general.
For all other $z$,
the complementary slackness implies the corresponding $\lambda$ element is 0.
As a result,
the corresponding diagonal entry in the bottom-right block of $S$ is nozero,
while the corresponding row in the bottom-left block of $S$ is straight 0.
So the corresponding entry in $\Delta_\lambda$ in \eqref{eq:grad_c_bridge} plays no role, and can be pruned.
In other words,
all $z \in S$ such that $\nbr{u}^2 + (z^\top u)^2 < 1$ can be treated as nonexistent.

The emprirical loss depends on $f(x_j)$,
which can be computed by $\iota(f)^\top k_j$.
Since $\iota(f) = (u^\top \sum\nolimits_j c_j k_j) u$, 
\eqref{eq:grad_c_dudc} allows us to backpropagate the gradient in $\iota(f)$ into the grdient in $\{c_j\}$.

\section{Experiments}
\label{app:experiment}

\input{appexpMixup}

\input{appexpML}

\end{document}

%% file: experiment.tex

\section{Experiments}

Here we highlight the major results and experiment setup.
Details on data preprocessing, experiment setting, optimization, and
additional results are given in Appendix \ref{app:experiment}.

\subsection{Sanity check for \sip\ based methods}

Our first experiment tries to test the effectiveness of optimizing the regularized risk \eqref{eq:obj_rrm_rkbs} with respect to the dual coefficients $\{c_j\}$ in \eqref{eq:representer_thm}. We compared 4 algorithms: 
{\sf SVM} with Gaussian kernel; 
{\sf Warping} which incorporates transformation invariance by kernel warping as described in \citet{MaGanZha19};
{\sf Dual} which trains the dual coefficients $\{c_j\}$ by LBFGS to minimize empirical risk as in \eqref{eq:obj_rrm_rkbs};
{\sf Embed} which finds the Euclidean embeddings by convex optimization as in \eqref{eq:po_equivalent}, followed by a linear classifier. 
The detailed derivation of the gradient in $\{c_j\}$ for {\sf Dual} is relegated to Appendix \ref{app:gradient_c}.

\renewcommand{\arraystretch}{1.2}
\begin{table}[t]
\caption{Test accuracy of minimizing empirical risk on binary classification tasks.}
\label{tab:opt_c}
\centering
\begin{tabular}{l|c|c|c|l}
\hline
         & {\sf SVM} & {\sf Warping} & {\sf Dual} & {\sf Embed} \\ \hline
4 v.s. 9 & 97.1         & 98.0           & 97.6           & 97.8                \\
2 v.s. 3 & 98.4         & 99.1           & 98.7           & 98.9                \\ \hline
\end{tabular}
\end{table}

Four transformation invariances were considered,
including rotation, scaling, and shifts to the left and upwards.
{\sf Warping} summed up the square of $\frac{\partial}{\partial \alpha} |_{\alpha=0} f(I(\alpha))$ over the four transformations,
while {\sf Dual} and {\sf Embed} took their max as the $R(f)^2$.
To ease the computation of derivative,
we resorted to finite difference for all methods,
with two pixels for shifting,
10 degrees for rotation, and 0.1 unit for scaling.
No data augmentation was applied.

All algorithms were evaluated on two binary classification tasks:
4 v.s. 9 and 2 v.s. 3, both sampling 1000 training and 1000 test examples from the MNIST dataset.

Since the square loss on the invariances used by {\sf Warping} makes good sense, 
the purpose of this experiment is \emph{not} to show that the \sip\ based methods are better in this setting.
Instead we aim to perform a sanity check on
a) good solutions can be found for the nonconvex optimization over the dual variables in {\sf Dual},
b) the Euclidean embedding of \sip\ representers performs competitively.
As Table \ref{tab:opt_c} shows,
both checks turned out affirmative,
with {\sf Dual} and {\sf Embed} delivering similar accuracy as {\sf Warping}.
In addition, {\sf Embed} achieved higher accuracy than dual optimization,
suggesting that the learned representations have well captured the invariances and possess better predictive power.

\input{expMixup}
\input{expML}

%% file: expMixup.tex
\subsection{Mixup}
\label{sec:exp_mixup}
We next investigated the performance of \sipsf\ on mixup.

\begin{table*}[t!]
\caption{Test accuracy on mixup classification task based on 10 random runs.}
\label{tab:mixup}
\centering
\scalebox{0.94}{
\begin{tabular}{r|c|ccc|ccc}
\hline
Dataset      &         & \multicolumn{3}{c|}{$n = 500$} & \multicolumn{3}{c}{$n = 1000$} \\ \hline
             & $p$       & $n$        & $2n$      & $4n$      & $n$        & $2n$      & $4n$      \\ \hline
MNIST        & \vanilla & $90.16\scriptstyle\pm1.40$      & $90.93\scriptstyle\pm1.01$     & $91.40\scriptstyle\pm1.04$     & $91.00\scriptstyle\pm1.17$      & $92.01\scriptstyle\pm1.21$     & $92.48\scriptstyle\pm1.03$     \\ \cline{2-8}
& \sipsf  & $\mathbf{91.36\scriptstyle\pm1.41}$      & $\mathbf{91.90\scriptstyle\pm1.08}$     & $\mathbf{92.11\scriptstyle\pm1.01}$     & $\mathbf{92.51\scriptstyle\pm1.01}$      & $\mathbf{92.79\scriptstyle\pm0.98}$     & $\mathbf{93.03\scriptstyle\pm1.00}$     \\ \hline

USPS         & \vanilla  & $90.54\scriptstyle\pm1.28$      & $91.76\scriptstyle\pm1.14$     & $92.40\scriptstyle\pm1.25$     & $93.87\scriptstyle\pm1.19$      & $94.72\scriptstyle\pm1.12$     & $95.32\scriptstyle\pm1.13$     \\ \cline{2-8}
             & \sipsf  & $\mathbf{92.46\scriptstyle\pm1.24}$      & $\mathbf{93.02\scriptstyle\pm1.12}$     & $\mathbf{93.21\scriptstyle\pm1.14}$     & $\mathbf{94.74\scriptstyle\pm0.97}$      & $\mathbf{95.11\scriptstyle\pm0.94}$     & $\mathbf{95.67\scriptstyle\pm0.96}$     \\ \hline
             
Fashion MNIST & \vanilla & $79.37\scriptstyle\pm3.11$     & $81.15\scriptstyle\pm2.08$    & $81.72\scriptstyle\pm1.96$    & $82.53\scriptstyle\pm1.49$     & $83.13\scriptstyle\pm1.36$    & $83.69\scriptstyle\pm1.31$    \\ \cline{2-8}
             & \sipsf  & $\mathbf{81.56\scriptstyle\pm2.27}$     & $\mathbf{82.16\scriptstyle\pm1.56}$    & $\mathbf{82.52\scriptstyle\pm1.49}$    & $\mathbf{83.28\scriptstyle\pm1.48}$     & $\mathbf{84.07\scriptstyle\pm1.32}$    & $\mathbf{84.34\scriptstyle\pm1.31}$    \\ \hline
\end{tabular}}

\vspace{1.5em}

  \caption{Test accuracy on  multilabel prediction with logic relationship}
  \label{tab:mlsvm}
  \centering
  %

\scalebox{0.96}{
   \begin{tabular}{l|ccc|ccc|ccc}
    \hline
    \multirow{2}{*}{Dataset}
            & \multicolumn{3}{c|}{\sipsf} & \multicolumn{3}{c|}{\mlsvm} & \multicolumn{3}{c}{\hrsvm}                                                                                 \\ \cline{2-10}
            & $ 100 $                     & $ 200 $                     & $ 500 $                    & $ 100 $  & $ 200 $  & $ 500 $  & $ 100 $  & $ 200 $         & $ 500 $         \\ \hline
    Enron   & $\mathbf{96.2} \dev{0.3} $             & $\mathbf{95.7} \dev{0.2}$             & $\mathbf{94.7} \dev{0.2}$            & $ 92.7 \dev{0.4}$ & $ 91.8 \dev{0.4} $ & $ 91.0 \dev{0.3} $ & $ 93.1 \dev{0.3}$ & $ 92.5 \dev{0.3}$  & $ 92.0 \dev{0.2} $
    \\ \hline
    Reuters & $\mathbf{95.7} \dev{1.4}$             & $\mathbf{97.2}\dev{1.2}$             & $\mathbf{98.0} \dev{0.4}$            & $ 94.2 \dev{1.4} $ & $ 95.1 \dev{1.3} $ & $ 95.2 \dev{1.2} $ & $ 95.1 \dev{1.2}$ & $ \mathbf{97.3} \dev{1.3} $        & $ 97.7 \dev{1.3}$                                        \\
    \hline
    WIPO    & $\mathbf{98.6}\dev{0.1}$             & $ \mathbf{98.4}\dev{0.1} $                    & $ 98.4 \dev{0.1} $                   & $ 98.1 \dev{0.3}$ & $ 98.2 \dev{0.2} $ & $ 98.3 \dev{0.1} $ & $ 98.3 \dev{0.1} $ & $\mathbf{98.5} \dev{0.1}$ & $\mathbf{98.7} \dev{0.2}$ \\
    \hline
  \end{tabular}}
\end{table*}

\vspace{-0.8em}
\paragraph{Datasets.}

We experimented with three image datasets: 
MNIST, USPS, and Fashion MNIST, each containing 10 classes. 
From each dataset, we drew $n \in \{500, 1000\}$ examples for training and $n$ examples for testing.
Based on the training data,
$p$ number of pairs were drawn from it.

Both \vanilla\ and \sipsf\ used Gaussian RKHS, 
along with Nystr\"om approximation whose landmark points consisted of the entire training set.
The vanilla mixup optimizes the objective \eqref{eq:mixup_vanilla} averaged over all sampled pairs. 
Following \citet{ZhaCisDauLop18}, 
The $\lambda$ was generated from a Beta distribution,
whose parameter was tuned to optimize the performance.  
Again, \sipsf\ was trained with a linear classifier.

\vspace{-0.75em}
\paragraph{Algorithms.}

We first ran mixup with stochastic optimization where pairs were drawn on the fly.
Then we switched to batch training of mixup (denoted as \vanilla),
with the number of sampled pair increased from $p = n$, $2n$, up to $5n$.
It turned out when $p = 4n$, the performance already matches the best test accuracy of the online stochastic version,
which generally witnesses much more pairs.
Therefore we also varied $p$ in $\{n, 2n, 4n\}$ when training \sipsf. each setting was evaluated 10 times with randomly sampled training and test data. The mean and standard deviation are reported in Table \ref{tab:mixup}.

\vspace{-0.75em}
\paragraph{Results.}

As Table \ref{tab:mixup} shows, 
\sipsf\ achieves higher accuracy than \vanilla\ on almost all datasets
and combinations of $n$ and $p$.
The margin tends to be higher when the training set size ($n$ and $p$) is smaller.
Besides, \vanilla\ achieves the highest accuracy at $p=4n$.


%% file: expML.tex
\subsection{Structured multilabel prediction}
\label{sec:expML}
Finally, we validate the performance of \sipsf\ on \textit{structured multilabel prediction} as described in Section
\ref{sec:multilabel},
showing that it is able to capture the structured relationships between the
class labels (implication and exclusion) in a hierarchical multilabel prediction task.

\vspace{-0.75em}
\paragraph{Datasets.}

We conducted experiments on three multilabel datasets where additional information is available about the hierarchy in its class labels \citep{Multilabel_dataset}:
Enron \cite{KliYan2004},
WIPO \cite{RouCraSaw06},
Reuters \cite{Davidetal04}.
Implication constraints were trivially derived from the hierarchy,
and we took siblings (of the same parent) as \textit{exclusion} constraints.
For each dataset, we
experimented with $100/100, 200/200, 500/500$ randomly drawn train/test examples.

\vspace{-0.6em}
\paragraph{Algorithms.}

We compared \sipsf\ with two baseline algorithms for  multilabel classification:
a multilabel SVM with RBF kernel (\mlsvm),
and an SVM that incorporates the hierarchical label constraints (\hrsvm) \cite{VatKubSar12}.
No inference is conducted at test time,
such as removing violations of implications or exclusions known a priori.

\vspace{-0.6em}
\paragraph{Results.}

Table \ref{tab:mlsvm} reports the accuracy
on the three train/test splits for each of the datasets.
Clearly, \sipsf\ outperforms both the baselines
in most of the cases.

%% file: appexpMixup.tex
\subsection{Additional experimental results on mixup}
\paragraph{Results. }
We first present more detailed experimental results for the mixup learning. Following the algorithms described in Section \ref{sec:exp_mixup}, each setting was evaluated 10 times with randomly sampled training and test data. The mean and standard deviation are reported in Table \ref{tab:mixup}. Since the results of \sipsf\ and \vanilla\ have the smallest difference under $n=1000, p=4n$, for each dataset, we show scatter plots of test accuracy under 10 runs  for this setting. In Figure \ref{fig:mxup_scatter}, the $x$-axis represents accuracy of \sipsf\ method, 
and the $y$-axis represents the accuracy of \vanilla. Obviously, most points fall above the diagonal, meaning \sipsf\ method outperforms \vanilla\  most of the time.
\vspace{-0.3cm}
\paragraph{Visualization.}
To show that \sipsf\ learned better representations in mixup, we next visualized the impact of the two different methods. 
Figure \ref{fig:loss_lambda} plots how the loss value of three randomly sampled pairs of test examples changes as a function of $\lambda$ in \eqref{eq:mixup_vanilla}. 
Each subplot here corresponds to a randomly chosen pair. 
By increasing $\lambda$ from 0 to 1 with a step size $0.1$, 
we obtained different mixup representations. We then applied the trained classifiers on these representations to compute the loss value. As shown in Figure \ref{fig:loss_lambda}, \sipsf\ always has a lower loss, especially when \vanilla\ is at its peak loss value. Recall in \eqref{eq:obj_mixup_kernel_embed}, \sipsf\ learns representations by considering the $\lambda$ that maximizes the change; this figure exactly verified this behavior and \sipsf\  learns better representation.
\vspace{-0.3cm}

%% file: appexpML.tex
\subsection{Additional experiments for structured multilabel prediction}
\vspace{-0.2cm}
Here, we provide more detailed results for our method applied to structured multilabel prediction, as described in Section \ref{sec:multilabel}.
\vspace{-0.3cm}
\paragraph{Accuracy on multiple runs.} We repeated the experiment, detailed in Section \ref{sec:expML} and tabulated in Table \ref{tab:mlsvm} ten times for all the three algorithms. Figures \ref{fig:scatter_reuters},\ref{fig:scatter_wipo},\ref{fig:scatter_enron} show the accuracy plot of our method (\sipsf)\ compared with baselines (\mlsvm\ and \hrsvm) on Enron \cite{KliYan2004},
WIPO \cite{RouCraSaw06},
Reuters \cite{Davidetal04} datasets with $100/100, 200/200, 500/500$ randomly drawn train/test examples over $10$ runs.
\vspace{-0.3cm}
\paragraph{Comparing constraint violations.}
In this experiment, we demonstrate the effectiveness of the model's ability to embed structures explicitly. Recall that for the structured multilabel prediction task, we wanted to incorporate two types of constraints (i) \textit{implication}, (ii) \textit{exclusion}. To test if our model (\sipsf)\ indeed learns representations that respect these constraints, we counted the number of test examples that violated the implication and exclusion constraints from the predictions. We repeated the test for \mlsvm\ and \hrsvm. 

We observed that \hrsvm\ and \sipsf\ successfully modeled implications on all the datasets. This is not surprising as \hrsvm\ takes the class hierarchy into account. The exclusion constraint, on the other hand, is a ``derived'' constraint and is not directly modeled by \hrsvm. Therefore, on datasets where \sipsf\ performed significantly better than \hrsvm, we might expect fewer exclusion violations by \sipsf\ compared to \hrsvm. 
To verify this intuition,
we considered the Enron dataset with $200/200$ train/test split where \sipsf\ performed better than \hrsvm. The constraint violations are shown as a line plot in Figure \ref{fig:const_vio}, with the constraint index on the $x$-axis and number of examples violating the constraint on the $y$-axis. 

Recall again that predictions in \sipsf\ for multilabel prediction are made using a linear classifier. Therefore the superior performance of \sipsf\, in this case, can be attributed to accurate representations learned by the model.

\clearpage

\begin{figure*}[t]
	\centering
		\includegraphics[width=0.9\linewidth]{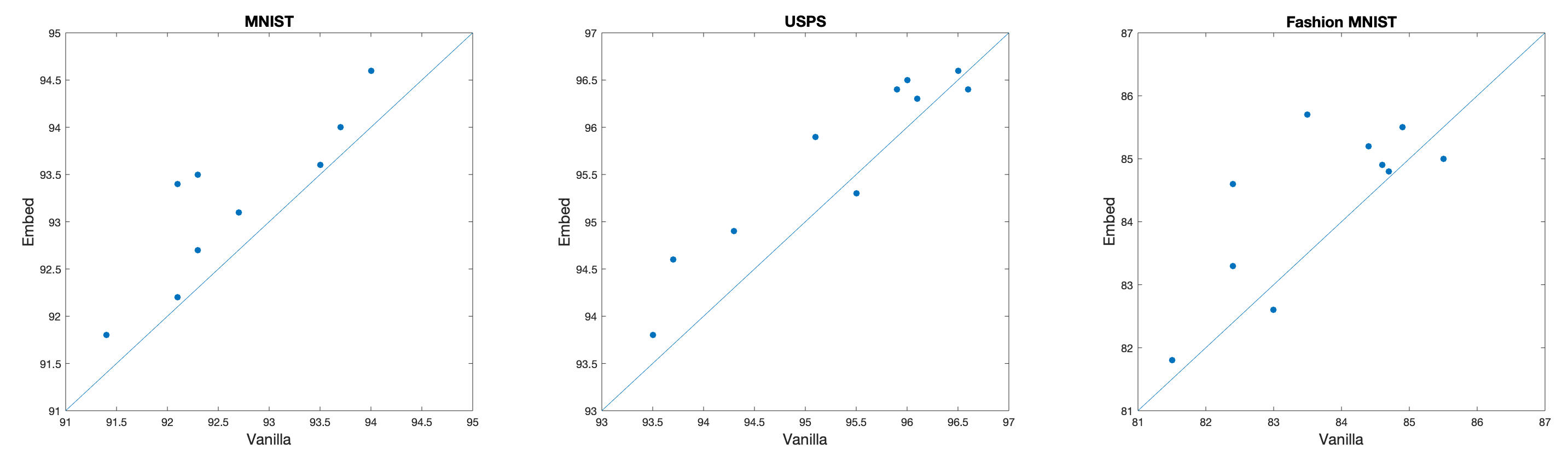}
		
		\caption{Scatter plot of test accuracy for mixup: $n=1000$, $p=4n$}
		\label{fig:mxup_scatter}
\end{figure*}

\begin{figure*}[t]
	\centering
		\includegraphics[width=0.9\linewidth]{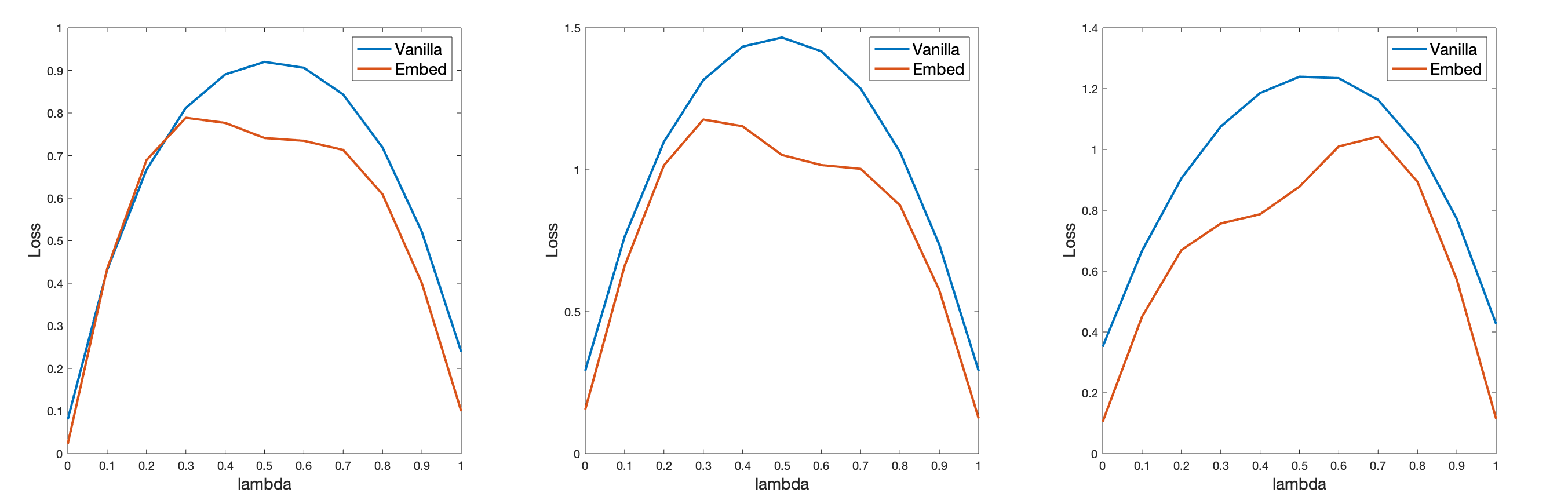}
		
		\caption{Plots of three different pairs of test examples, showing how loss values change as a function of $\lambda$}
		\label{fig:loss_lambda}
\end{figure*}

\begin{figure*}[ht]
    \subfloat[$100/100$ train/test split]{%
       \includegraphics[width=0.30\textwidth]{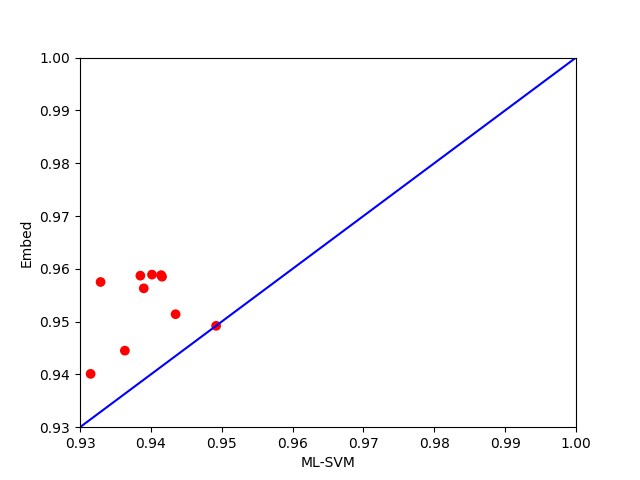}
     }\hfill
     \subfloat[ $200/200$ train/test split]{%
       \includegraphics[width=0.30\textwidth]{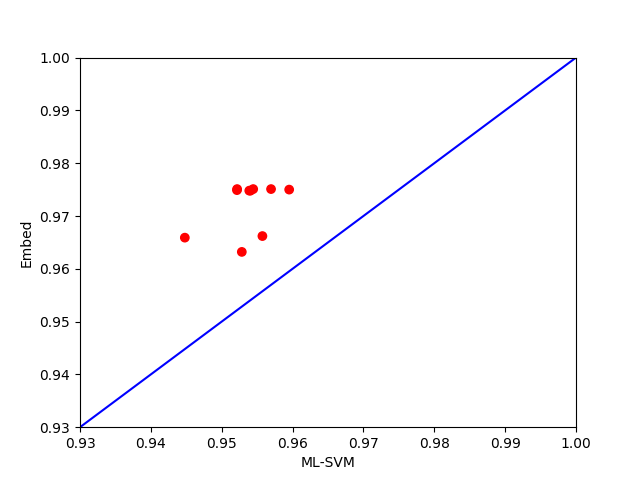}
     }\hfill
    \subfloat[$500/500$ train/test split]{%
       \includegraphics[width=0.30\textwidth]{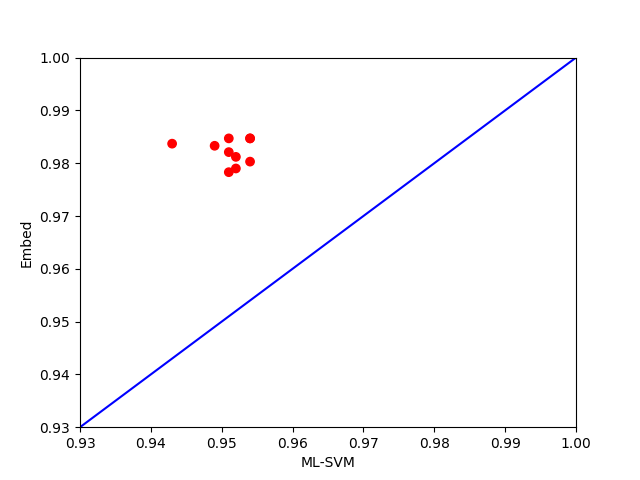}
     }\hfill \\
     \subfloat[$100/100$ train/test split]{%
       \includegraphics[width=0.30\textwidth]{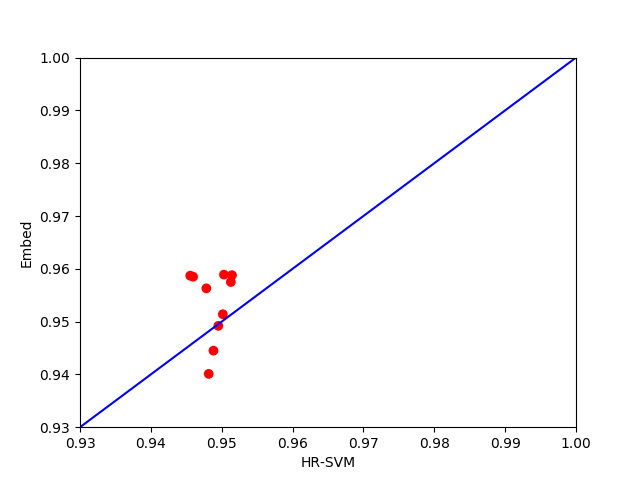}
     }\hfill
     \subfloat[$200/200$ train/test split]{%
       \includegraphics[width=0.30\textwidth]{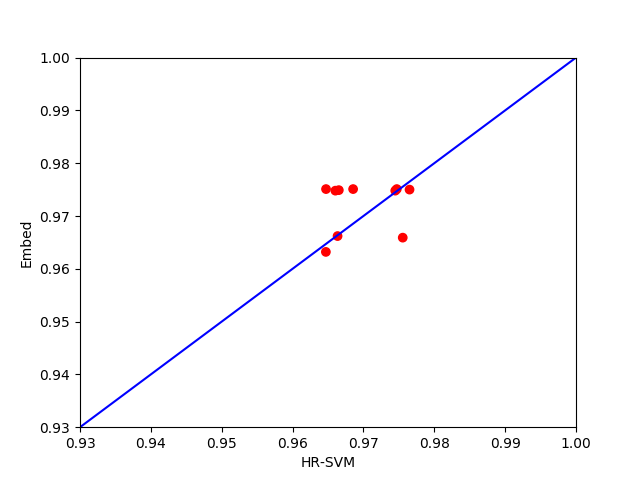}
     }\hfill
    \subfloat[ $500/500$ train/test split]{%
       \includegraphics[width=0.30\textwidth]{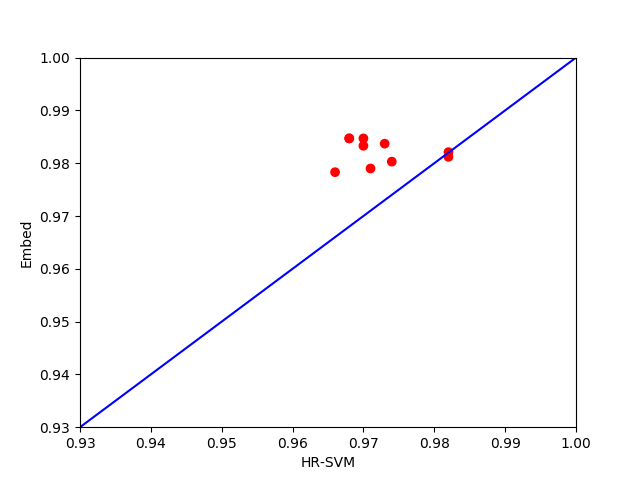}
     }\hfill
    \caption{Test accuracy of \mlsvm\ vs \sipsf\ (top row) and \hrsvm\ vs \sipsf\ (bottom row)  $10$ runs on the Reuters dataset}
    \label{fig:scatter_reuters}
\end{figure*}

\begin{figure*}[ht]
    \subfloat[$100/100$ train/test split]{%
       \includegraphics[width=0.3\textwidth]{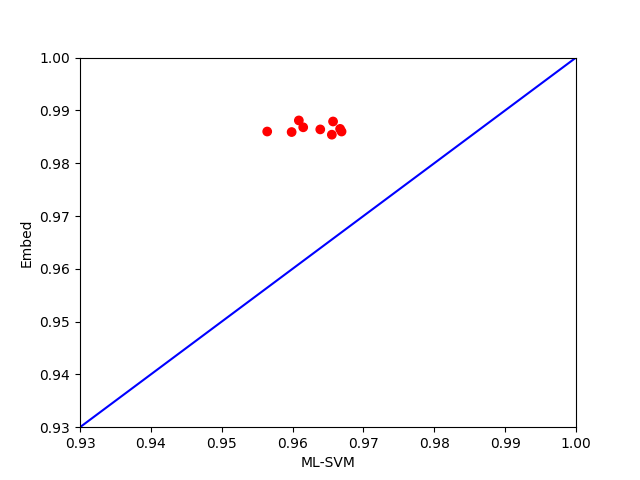}
     }\hfill
     \subfloat[ $200/200$ train/test split]{%
       \includegraphics[width=0.3\textwidth]{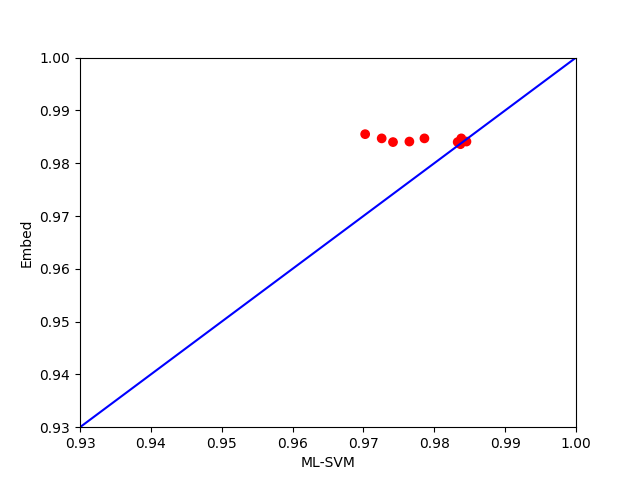}
     }\hfill
    \subfloat[$500/500$ train/test split]{%
       \includegraphics[width=0.3\textwidth]{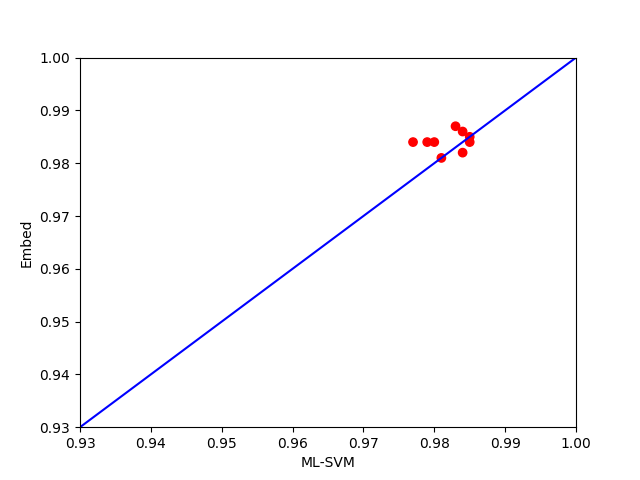}
     }\hfill \\
     \subfloat[$100/100$ train/test split]{%
       \includegraphics[width=0.3\textwidth]{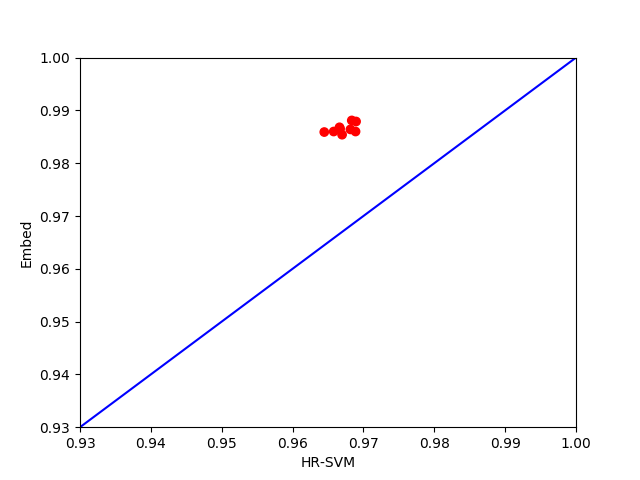}
     }\hfill
     \subfloat[$200/200$ train/test split]{%
       \includegraphics[width=0.3\textwidth]{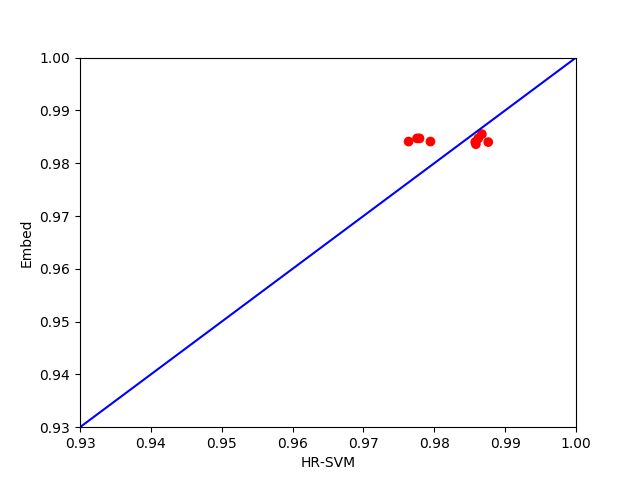}
     }\hfill
    \subfloat[ $500/500$ train/test split]{%
       \includegraphics[width=0.3\textwidth]{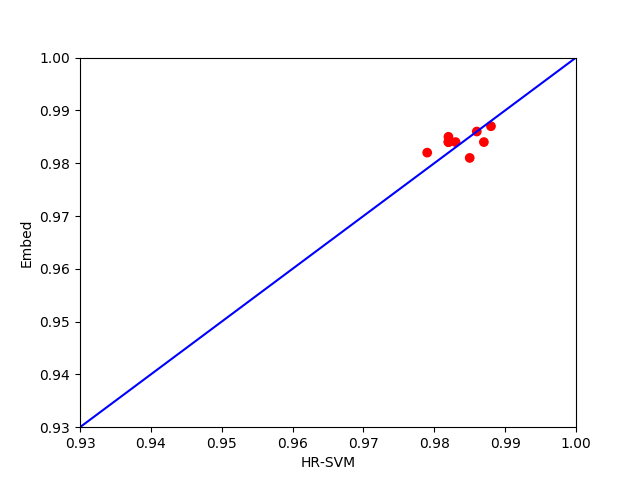}
     }\hfill
    \caption{Test accuracy of \mlsvm\ vs \sipsf\ (top row) and \hrsvm\ vs \sipsf\ (bottom row)  $10$ runs on the WIPO dataset}
    \label{fig:scatter_wipo}
\end{figure*}
\begin{figure*}[ht]
    \centering
    \subfloat[$100/100$ train/test split]{%
       \includegraphics[width=0.3\textwidth]{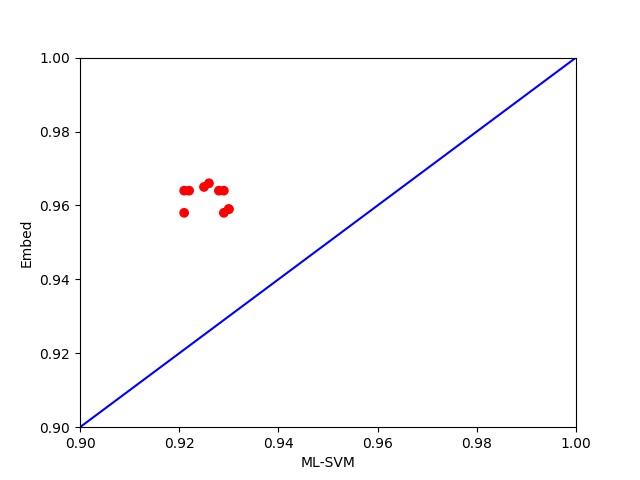}
     }\hfill
     \subfloat[ $200/200$ train/test split]{%
       \includegraphics[width=0.3\textwidth]{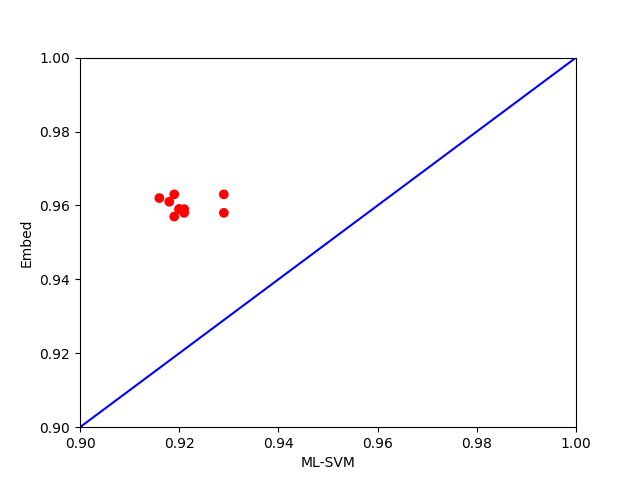}
     }\hfill
    \subfloat[$500/500$ train/test split]{%
       \includegraphics[width=0.3\textwidth]{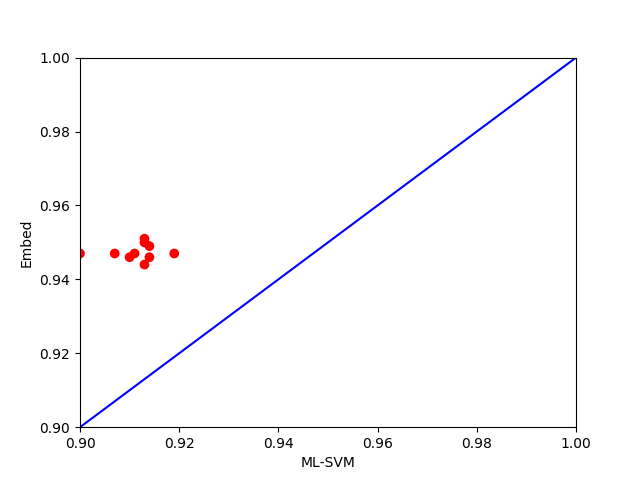}
     }\hfill \\
     \subfloat[$100/100$ train/test split]{%
       \includegraphics[width=0.3\textwidth]{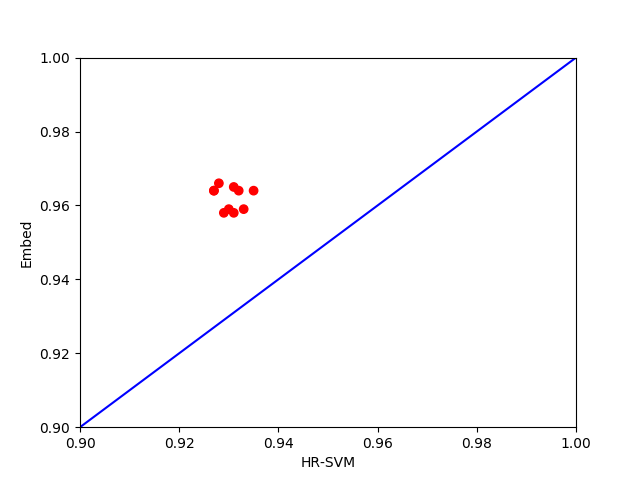}
     }\hfill
     \subfloat[$200/200$ train/test split]{%
       \includegraphics[width=0.3\textwidth]{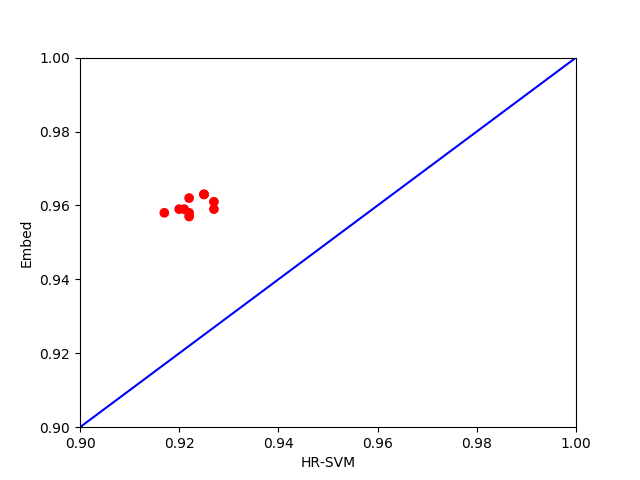}
     }\hfill
    \subfloat[ $500/500$ train/test split]{%
       \includegraphics[width=0.3\textwidth]{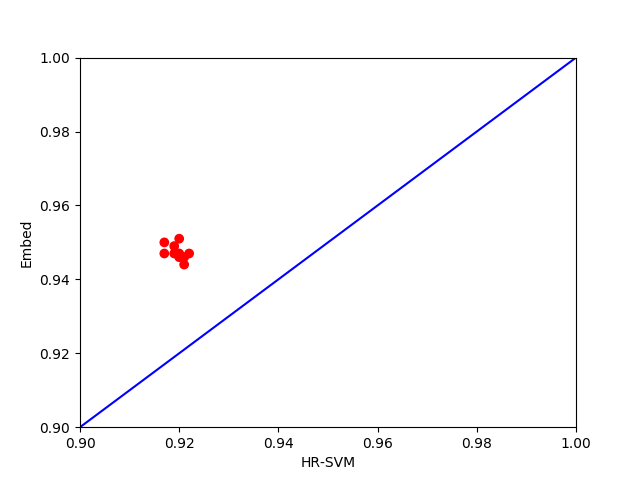}
     }\hfill
    \caption{Test accuracy of \mlsvm\ vs \sipsf\ (top row) and \hrsvm\ vs \sipsf\ (bottom row)  $10$ runs on the ENRON dataset}
    \label{fig:scatter_enron}
\end{figure*}
\begin{figure*}[t]
    \centering
    \includegraphics[scale=0.8]{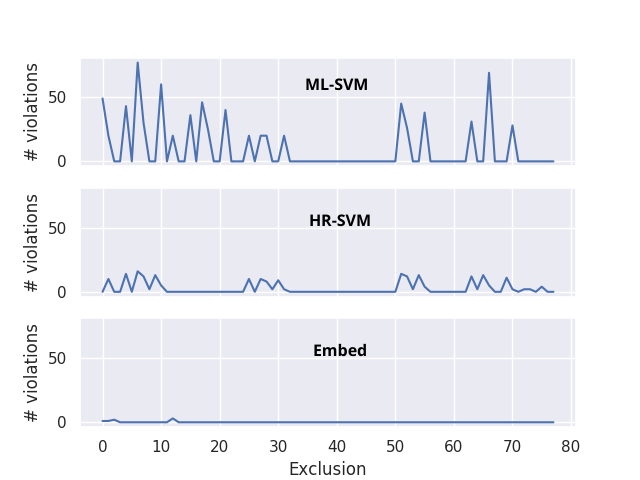}
    \caption{The number of violations for each exclusion constraint on the test set by (from top) \mlsvm, \hrsvm, and \sipsf\ on the Enron dataset with $200/200$ train/test examples. }
    \label{fig:const_vio}
\end{figure*}
\begin{figure*}[t]
    \centering
    \includegraphics[scale=0.8]{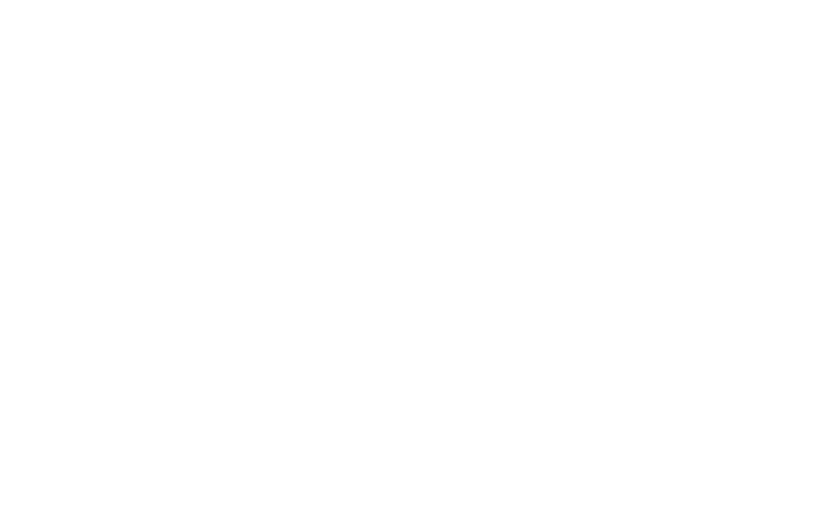}
\end{figure*}